\documentclass{article}

\usepackage{color}
\usepackage[numbers]{natbib}

\usepackage{enumitem}

\usepackage{amsmath}
\usepackage{amssymb}
\usepackage{mathrsfs}
\usepackage{bm}
\usepackage{amsthm}

\newtheorem{theorem}{Theorem}[subsection]
\newtheorem{lemma}[theorem]{Lemma}
\newtheorem{corollary}[theorem]{Corollary}
\newtheorem{proposition}[theorem]{Proposition}

\newtheorem{assumption}{Assumption}

\usepackage{PRIMEarxiv}

\usepackage[utf8]{inputenc} 
\usepackage[T1]{fontenc}    
\usepackage{hyperref}       
\usepackage{url}            
\usepackage{booktabs}       
\usepackage{amsfonts}       
\usepackage{nicefrac}       
\usepackage{microtype}      
\usepackage{lipsum}
\usepackage{fancyhdr}       
\usepackage{graphicx}       
\graphicspath{{media/}}     

\title{The phase diagram of kernel interpolation in large dimensions
}

\author{
  Haobo Zhang, Weihao Lu \thanks{Co-first author}\\
  Center for Statistical Science, Department of Industrial Engineering, Tsinghua University, \\
  Beijing, China \\
  \texttt{\{zhang-hb21, luwh19\}@mails.tsinghua.edu.cn} \\
   \And
  Qian Lin \thanks{Corresponding author} \\
  Center for Statistical Science, Department of Industrial Engineering, Tsinghua University, \\
  Beijing, China \\
  \texttt{qianlin@tsinghua.edu.cn} \\
}

\usepackage{hyperref}  
\hypersetup{hidelinks}
\hypersetup{
colorlinks=true
}

\begin{document}
\maketitle

\begin{abstract}
 The generalization ability of kernel interpolation in large dimensions (i.e., $n \asymp d^{\gamma}$ for some $\gamma>0$) might be one of the most interesting problems in the recent renaissance of kernel regression, since it may help us understand the ‘benign overfitting phenomenon’ reported in the neural networks literature. Focusing on the inner product kernel on the sphere, we fully characterized the exact order of both the variance and bias of large-dimensional kernel interpolation under various source conditions $s\geq 0$.  Consequently, we obtained the $(s,\gamma)$-phase diagram of large-dimensional kernel interpolation, i.e., we determined the regions in $(s,\gamma)$-plane where the kernel interpolation is minimax optimal, sub-optimal and  inconsistent.
\end{abstract}


\section{Introduction}

The `benign overfitting phenomenon' of neural networks (or other models with enough capacity) has been widely observed in \cite[etc.]{belkin2019_ReconcilingModern,zhang2021understanding}. It indicates that over-parameterized neural networks can interpolate noisy data but still generalize, which contrasts with the traditional bias-variance trade-off in statistics. In recent years, the neural tangent kernel theory has established a bridge between the training dynamics of infinite-width neural networks and kernel regression \cite[etc.]{jacot2018neural,du2018_GradientDescent,allen2019convergence}. Thus, the generalization ability of \textit{kernel interpolation} (formally defined in Eq.\eqref{eq:KMinNorm}) has been considered as one of the cornerstones towards a theoretical explanation of the ‘benign overfitting phenomenon’.

In the fixed-dimensional setting, several works showed the inconsistency of kernel interpolation \cite{rakhlin2019consistency,beaglehole2022_KernelRidgeless,buchholz2022_KernelInterpolation,lai2023generalization,Li2023KernelIG}. In addition, \cite{mallinar2022benign,cheng2024characterizing} discussed the so called `tempered' and `catastrophic' regimes of kernel interpolation, and \cite{haas2024mind} found that adding spike components to the kernels could lead to benign overfitting in fixed dimensions. 

In the large-dimensional setting, common assumptions on the eigenvalues of the kernel (e.g., polynomial decay, exponential decay) no longer hold, thus the answer could be much more complicated. A line of work \cite{Ghorbani2019LinearizedTN,Donhauser_how_2021,mei2022generalization,xiao2022precise,misiakiewicz_spectrum_2022,hu2022sharp} considered the square-integrable function space and showed the `polynomial approximation barrier' of kernel regression (kernel regression is consistent if and only if the true function is a polynomial with a low-degree polynomial). In addition, supposing that the true function has some nice properties, \cite{Liang_Just_2019,liang2020multiple,aerni2022strong,barzilai2023generalization} showed that kernel interpolation can generalize. Although researchers have known that kernel interpolation can indeed generalize in large dimensions, further questions have been rarely studied: whether the generalization error converges fast, or more precisely, is kernel interpolation minimax optimal?

In this paper, we study the exact convergence rates (both upper and lower bounds) of the generalization error of kernel interpolation in the large-dimensional setting, where the samples size $n$ and the dimension $d$ satisfy $n \asymp d^{\gamma}$ for some $\gamma >0$. We focus on the inner product kernel on the sphere, whose eigenvalues and eigenfunctions are well-studied. In addition, we use the `source condition $s$' to characterize the smoothness of the true function. For any fixed $\gamma>0, s \ge 0$, we provide matching upper and lower bounds for both the variance and bias term of kernel interpolation. Based on these bounds, we show that kernel interpolation  is inconsistent when $s=0$ or $ \gamma \in \mathbb{N}^{+} $; and is consistent when $s>0$ and $ \gamma \in (0,\infty) \backslash \mathbb{N}^{+}$. Furthermore, we find a threshold $\Gamma(\gamma)$ (please refer to equation \eqref{eqn:Gamma-gamma}) and prove that when $s > \Gamma(\gamma)$, kernel interpolation is sub-optimal; when $ 0 < s \le \Gamma(\gamma)$, kernel interpolation is minimax optimal. We summarize them into the $(s,\gamma)-$phase diagram in Figure \ref{fig1}. To our knowledge, this is the first result fully characterizing the generalization ability of kernel interpolation in large-dimensional setting.


\section{Preliminaries}

\subsection{Basic settings}
Let $\mathcal{X} = \mathbb{S}^{d} \subset \mathbb{R}^{d+1}$ (i.e., the unit sphere ) be the input space, $ \mathcal{Y} \subseteq \mathbb{R}$ be the output space and $\mu = \mu_{d}$ be the uniform distribution on $ \mathcal{X} $. We consider the large-dimensional setting, where the sample size $n\asymp d^{\gamma}$ for some $\gamma >0$. Suppose that the $n$ samples $\{ (x_{i}, y_{i}) \}_{i=1}^{n}$ are i.i.d. sampled from the model 
\begin{equation}\label{main data model}
    y = f^{*}(x) + \epsilon,  \quad x \sim \mu,
\end{equation}
where $ f^{*} = f^{*}_{d} $ is the true function and the noise $\epsilon$ has zero mean and variance $\sigma^{2} > 0$. Throughout the paper, we consider the separable reproducing kernel Hilbert space $\mathcal{H} = \mathcal{H}_{d}$ associated with the following inner product kernel $k = k_{d}$ on $\mathcal{X}$:

\begin{assumption}[Inner product kernel]\label{assumption inner product kernel}
    Suppose that the kernel function $k(\cdot,\cdot) = k_{d}(\cdot,\cdot): \mathcal{\mathcal{X}} \times \mathcal{\mathcal{X}} \to \mathbb{R}$ is an inner product kernel on the sphere satisfying:
    \begin{displaymath}
        k_{d}(x, x^{\prime}) = \Phi\left( \langle x, x^{\prime} \rangle\right), ~~ x, x^{\prime} \in \mathbb{S}^{d},
    \end{displaymath}
    where $ \Phi(t) \in \mathcal{C}^{\infty} \left([-1,1]\right)$ is a fixed function independent of $d$ and 
    \begin{displaymath}
        \Phi(t) = \sum_{j=0}^\infty a_j t^j, ~ a_{j} > 0, ~\text{for any}~ j = 0, 1, 2,\dots.
    \end{displaymath}

\end{assumption}

Without loss of generality, we assume $ \sum_{j=0}^{\infty} a_{j} \le 1$ and $a_{j} >0, ~\text{for any}~ j=0,1,\ldots$ throughout this paper. Thus we have $ \sup_{x \in \mathcal{X}} k(x,x) \le 1$. The proof can be easily extended to inner product kernels which have infinite positive numbers in $\{a_{j}\}_{j=0}^{\infty}$ (e.g., the neural tangent kernel \cite{liang2020multiple,zhang2024optimal}). We assume $\Phi(t)$ (or $ \{a_{j}\}_{j=0}^{\infty}$) to be fixed and ignore the dependence of constants on it. In the rest of this paper, we frequently omit the subscript $d$ for brevity of notations.

As discussed in the introduction, the inner product kernel has attracted much interest \cite{liang2020multiple,Ghorbani2019LinearizedTN,misiakiewicz_spectrum_2022,xiao2022precise,zhang2024optimal}, etc.. This is partly due to the concise characterization of its eigenvalues and eigenfunctions. We refer to Section \ref{subsection eigen of inner property} for more details. Although recent works \cite{mei2022generalization,barzilai2023generalization} discuss the results under general assumptions on the eigenvalues and eigenfunctions, concrete convergence rates in large dimensions have only been proved when $\mathcal{X}$ is the sphere or discrete hypercube (\cite{aerni2022strong}). To avoid complicated notations and non-intuitive technical conditions (such as hypercontractivity in \cite{mei2022generalization}, etc.) , we prefer to work with inner product kernel on the sphere so that we can obtain a concise phase diagram and leave the more technical works to the future work.

\subsection{Kernel interpolation estimator}
Kernel ridge regression, one of the most popular kernel methods, constructs an estimator $\hat{f}_{\lambda}$ by solving the penalized least square problem
\begin{equation}\label{eq def krr}
    \hat{f}_\lambda = \underset{f \in \mathcal{H}}{\arg \min } \left(\frac{1}{n} \sum_{i=1}^n\left(y_i-f\left( x_{i} \right)\right)^2+ \lambda \|f\|_{\mathcal{H}}^2\right),
\end{equation}
where $\lambda > 0$ is referred to as the regularization parameter. Classical statistical analyses of kernel ridge regression can be found in \cite[etc.]{Caponnetto2007OptimalRF,steinwart2008support}. In the limiting case $(\lambda \to 0)$, we obtain the (minimum-norm) \textit{kernel interpolation} estimator, which is the main interest of this paper:
\begin{align}
    \label{eq:KMinNorm}
    \hat{f}_{\mathrm{inter}} = \hat{f}_{0} = \mathop{\arg\min}\limits_{f \in \mathcal{H}} \| f \|_{\mathcal{H}} ~~~~ \text{subject to} ~~  f(x_i) = y_i, \quad i = 1,\dots,n.
\end{align}


Denote the bias-variance decomposition of the generalization error of $\hat{f}_{\mathrm{inter}}$ as 
\begin{align}\label{eq bias var decomposition}
    E_{x, \epsilon} \left[ \left(\hat{f}_{\mathrm{inter}}(x) - f^{*}(x) \right)^{2} \right] = \mathrm{bias}^2(\hat{f}_{\mathrm{inter}}) + \mathrm{var}(\hat{f}_{\mathrm{inter}}),
\end{align}
where $E_{x, \epsilon}$ means taking expectation on the new sample $x$ and the random noise in the training samples, and
\begin{equation}\label{eq bias term}
    \mathrm{bias}^2(\hat{f}_{\mathrm{inter}}) = E_x\left[\left(E_\epsilon\left(\hat{f}_{\mathrm{inter}}(x)\right)-f^{*}(x)\right)^2\right],
\end{equation}
and
\begin{equation}\label{eq var term}
    \mathrm{var}(\hat{f}_{\mathrm{inter}}) = E_{x, \epsilon}\left[\left(\hat{f}_{\mathrm{inter}}(x)-E_\epsilon\left(\hat{f}_{\mathrm{inter}}(x)\right)\right)^2\right].
\end{equation}

Throughout the paper, we use asymptotic notations $O(\cdot),~o(\cdot),~\Omega(\cdot),~\omega(\cdot)$ and $\Theta(\cdot)$.
We also write $a_n \asymp b_n$ for $a_n = \Theta(b_n)$; $a_n \lesssim b_n$ for $a_n = O(b_n)$; $a_n \gtrsim b_n$ for $a_n = \Omega(b_n)$; $ a_n \ll b_n$ for $a_n = o(b_n)$ and $a_n \gg b_n$ for $a_n = \omega(b_n) $. We will also use the probability versions of the asymptotic notations such as $O_{\mathbb{P}}(\cdot), o_{\mathbb{P}}(\cdot), \Omega_{\mathbb{P}}(\cdot), \omega_{\mathbb{P}}(\cdot), \Theta_{\mathbb{P}}(\cdot)$. For instance, we say the random variables $ X_{n}, Y_{n} $ satisfying $ X_{n} =  O_{\mathbb{P}}(Y_{n}) $ if and only if for any $\epsilon > 0$, there exist a constant $C_{\epsilon} $ and $ N_{\epsilon}$ such that $ \mathrm{pr}\left( |X_{n}| \ge C_{\epsilon} |Y_{n}| \right) \le \epsilon, ~\text{for any}~ n > N_{\epsilon}$. We use $L^{p}(\mathcal{X}, \mu)$ (or $L^{p}$ in brevity) to represent the $L^{p}$ spaces. For a positive real number $ \alpha$, denote $\lfloor \alpha \rfloor$ as the largest integer not exceeding $\alpha$.

\subsection{Interpolation space of reproducing kernel Hilbert space}
\label{section interpolation space and rkhs}

In order characterize the `smoothness' of $f^{*}$, researchers usually assume $ f^{*} \in [\mathcal{H}]^{s}$, the interpolation space of the reproducing kernel Hilbert space. In this subsection, we provide some prior knowledge of the interpolation space.

We fix a continuous kernel function $k$ over
$\mathcal{X}$ and denote $\mathcal{H}$ as the separable reproducing kernel Hilbert space associated with $k$. Denote the natural embedding inclusion operator as $S_{k}: \mathcal{H} \to L^{2}$. Then its adjoint operator $S_{k}^{*}: L^{2} \to \mathcal{H}  $ is an integral operator, i.e., for $f \in L^{2}$ and $x \in \mathcal{X}$, we have 
\begin{displaymath}
\left(S_{k}^{*} f\right)(x)=\int_{\mathcal{X}} k\left(x, x^{\prime}\right) f\left(x^{\prime}\right) \mathrm{d} \mu\left(x^{\prime}\right).
\end{displaymath}
Supposing that $ \sup_{x\in \mathcal{X}} k(x,x) \le 1$ (as the case of inner product kernel in Assumption \ref{assumption inner product kernel}), $S_{k}$ and $S_{k}^{*}$ are Hilbert-Schmidt operators (thus compact) and the Hilbert-Schmidt norms (denoted as $\left\| \cdot \right\|_{\textrm{HS}}$) satisfy that $\left\| S_{k}^{*} \right\|_{\textrm{HS}} = \left\| S_{k} \right\|_{\textrm{HS}} \le 1.$ Next, we can define an integral operator: 
\begin{displaymath}
    L_{k}=S_{k} S_{k}^{*}: L^{2} \rightarrow L^{2}.
\end{displaymath}
It is well known \cite{steinwart2012_MercerTheorem} that $L_{k}$ is self-adjoint, positive-definite and trace class (thus Hilbert-Schmidt and compact), with trace norm $ \left\|L_{k}\right\|_{\textrm{TR}} =\left\|S_{k}^{*}\right\|_{\textrm{HS}}^2 \le 1$. 

The spectral theorem for self-adjoint compact operators yields that there is an at most countable index set $N$, a non-increasing summable sequence $\{ \lambda_{i} \}_{i \in N} \subseteq (0,\infty)$ and a family $\{ e_{i} \}_{i \in N} \subseteq \mathcal{H} $, such that $\{ e_{i} \}_{i \in N}$ is an orthonormal basis of $\overline{\operatorname{ran} S_{k}} \subseteq L^{2}$ and $\{ \lambda_{i}^{1/2} e_{i} \}_{i \in N}$ is an orthonormal basis of $\mathcal{H}$. Without loss of generality, we assume $|N| = \infty$ (as the case of inner product kernel) in the following. Then, the integral operator and the kernel function can be written as 
\begin{equation}\label{eq dec of lambda e}
    L_{k}=\sum_{i =1}^{\infty} \lambda_i\left\langle\cdot,e_{i} \right\rangle_{L^{2}}e_{i} \quad \text { and } \quad  k\left(x, x^{\prime}\right)=\sum_{i =1}^{\infty} \lambda_i e_i(x) e_i\left(x^{\prime}\right),
\end{equation}
where the convergence of kernel expansion is absolute and uniform for $x,x^{\prime}$. We refer to $\{ e_{i} \}_{i =1}^{\infty}$ and $\{ \lambda_{i} \}_{i =1}^{\infty}$ as the eigenfunctions and eigenvalues of $k$.

Decomposition Eq.\eqref{eq dec of lambda e} allows us to define the interpolation spaces (power spaces) of reproducing kernel Hilbert space. For any $ s \ge 0$, we define the fractional power integral operator $L_{k}^{s}: L^{2} \to L^{2}$ 
\begin{displaymath}
  L_{k}^{s}(f)=\sum_{i =1}^{\infty} \lambda_i^{s} \left\langle f, e_i\right\rangle_{L^2} e_i.
\end{displaymath}
Then the interpolation space (power space) $[\mathcal{H}]^s $ is defined as
\begin{displaymath}
  [\mathcal{H}]^s = \operatorname{Ran} L_{k}^{s/2} = \left\{\sum_{i =1}^{\infty} a_i \lambda_i^{s / 2}e_{i}: \left(a_i\right)_{i =1}^{\infty} \in \ell_2(N)\right\} \subseteq L^{2},
\end{displaymath}
equipped with the norm
\begin{displaymath}
    \left\| \sum\limits_{i=1}^{\infty} a_i \lambda_i^{s / 2} e_i\right\|_{[\mathcal{H}]^s} = \left(\sum\limits_{i=1}^{\infty} a_i^{2} \right)^{1 / 2 } .
\end{displaymath}
It is easy to show that $[\mathcal{H}]^s $ is also a separable Hilbert space with orthogonal basis $ \{ \lambda_{i}^{s/2} e_{i}\}_{i =1}^{\infty}$. Specially, we have $[\mathcal{H}]^0 \subseteq L^{2} $ and $[\mathcal{H}]^1 = \mathcal{H}$. For $0 < s_{1} < s_{2}$, the embeddings $ [\mathcal{H}]^{s_{2}} \hookrightarrow[\mathcal{H}]^{s_{1}} \hookrightarrow[\mathcal{H}]^0 $ exist and are compact \cite{fischer2020_SobolevNorm}. For the functions in $[\mathcal{H}]^{s}$ with larger $s$, we say they have higher regularity (smoothness) with respect to $k$ (or $\mathcal{H}$). As an example, the Sobolev space $H^{m}(\mathcal{X})$ is a reproducing kernel Hilbert space if $m > \frac{d}{2}$ and its interpolation space is still a Sobolev space given by $ [H^{m}(\mathcal{X})]^s \cong H^{m s}(\mathcal{X}), ~\text{for any}~ s>0 $.

For the inner product kernel considered in this paper (see Assumption \ref{assumption inner product kernel}), we have a more specific characterization of the eigenvalues and eigenfunctions: (see, e.g., \cite{smola2000_RegularizationDotproduct,10.5555/559923})
\begin{displaymath}
  k(x,x^{\prime}) = \sum_{k=0}^{\infty} \mu_k \sum_{m=1}^{N(d,k)} \psi_{k,m}(x) \psi_{k,m}(x^{\prime}),
\end{displaymath}
where $ \{\psi_{k,m}\}_{m=1}^{N(d,k)}$ are spherical harmonic polynomials of degree $k$; $ \mu_{k}>0$ are the eigenvalues with multiplicity $N(d,0)=1$; $N(d, k) = (2k+d-1) \cdot (k+d-2)! / \left[ k (d-1)!(k-1)! \right], k =1,2,\ldots$. We will frequently adopt the notations of $ \mu_{k}, \psi_{k,m}$ in the proof.

\section{Main results}\label{subsection gen of ki}

We first state a theorem about the variance term.
\begin{theorem}\label{theorem variance}
    Let $c_{1} d^{\gamma} \le n \le c_{2} d^{\gamma} $ for some fixed $ \gamma >0 $ and absolute constants $ c_{1}, c_{2}$. Consider model Eq.\eqref{main data model} and let $k = k_{d}$ be the inner product kernel on the sphere satisfying Assumption \ref{assumption inner product kernel}. Denote $l = \lfloor \gamma \rfloor$. Then the variance term Eq.\eqref{eq var term} of the kernel interpolation estimator $\hat{f}_{\mathrm{inter}}$ satisfies
    \begin{displaymath}
        \mathrm{var}(\hat{f}_{\mathrm{inter}}) = \sigma^{2} \cdot \Theta_{\mathbb{P}} \left( d^{l-\gamma} + d^{\gamma-l-1} \right),
    \end{displaymath}
    where $\Theta_{\mathbb{P}} $ only involves constants depending on $\gamma$, $c_{1}$ and $ c_{2}$. 
\end{theorem}

To state the theorem about the bias term, we need the following characterization of the smoothness of the true function with respect to $\mathcal{H}$.

\begin{assumption}[Source condition]\label{assumption source condition}
$ $
\begin{itemize}
    \item[(a)] Suppose that $ f^{*}(x) = \sum_{k=0}^{\infty} \sum_{m=1}^{N(d,k)} \theta^{*}_{k,m} \psi_{k,m}(x) \in [\mathcal{H}]^{s}$ for some $s \ge 0$ and the $[\mathcal{H}]^{s}$-norm satisfies
    \begin{equation}\label{assumption source part 1}
        \left\| f^{*} \right\|_{[\mathcal{H}]^{s}} = \sum\limits_{k=0}^{\infty} \mu_{k}^{-s} \sum\limits_{m =1}^{N(d,k)}  (\theta^{*}_{k,m})^{2} \le R_{\gamma}, 
    \end{equation}
    where $R_{\gamma}$ is a constant only depending on $\gamma$.
   \item[(b)] Denote $\gamma > 0$ as the scale in Theorem \ref{theorem variance} and $l = \lfloor \gamma \rfloor$. Suppose that for the $s$ above, there exists an absolute constant $c_{0} > 0$ such that for any $ d $, we have
   \begin{equation}\label{ass of fi}
        \sum\limits_{k=0}^{l} \sum\limits_{m =1}^{N(d,k)}  (\theta^{*}_{k,m})^{2} \ge c_{0};~~ \text{and} ~~\mu_{p}^{-s} \sum\limits_{m=1}^{N(d,p)}  (\theta^{*}_{p,m})^{2} \ge c_{0}, ~~p=l, l+1.
   \end{equation}
\end{itemize}

We refer to $s$ as the source condition of $f^{*}$.
\end{assumption}

Assumption \ref{assumption source condition} (a) is usually used as the traditional source condition assumption \cite{caponnetto2006optimal,fischer2020_SobolevNorm}, etc.. The purpose of Assumption \ref{assumption source condition} (b) is to obtain a reasonable lower bound of the bias term. It is equivalent to assume that the $L^{2}$-norm of the projection of $f^{*}$ on the first $ \lfloor \gamma \rfloor$ eigenspaces, as well as the $ [\mathcal{H}]^{s}$-norm of the projection on the $\lfloor \gamma \rfloor$th and $(\lfloor \gamma \rfloor + 1)$th eigenspaces is non-vanishing. In a word, Assumption \ref{assumption source condition} (a) and (b) together imply that $ f^{*} \in [\mathcal{H}]^{s}$ and $ f^{*} \notin [\mathcal{H}]^{t}$ for any $t > s$. 

Similar assumptions have been adopted when one is interested in the lower bound of the generalization error, e.g., Eq.(8) in \cite{Cui2021GeneralizationER} and Assumption 3 in \cite{NEURIPS2023_9adc8ada} for fixed-dimensional setting; Assumption 5 in \cite{zhang2024optimal} for large-dimensional setting.

\begin{theorem}\label{theorem bias}
Let $c_{1} d^{\gamma} \le n \le c_{2} d^{\gamma} $ for some fixed $ \gamma \in (0,\infty) \backslash \mathbb{N}^{+}$ and absolute constants $ c_{1}, c_{2}$. Consider model Eq.\eqref{main data model} and let $k = k_{d}$ be the inner product kernel on the sphere satisfying Assumption \ref{assumption inner product kernel}. Further suppose that the true function $f^{*} = f^{*}_{d}$ in model Eq.\eqref{main data model} satisfies Assumption \ref{assumption source condition} for some $s \ge 0$. Denote $l = \lfloor \gamma \rfloor$ and denote $ \tilde{s} = \min\{s,2\}$. Then the bias term Eq.\eqref{eq bias term} of the kernel interpolation estimator $\hat{f}_{\mathrm{inter}}$ satisfies 
\begin{displaymath}
    \mathrm{bias}^{2}(\hat{f}_{\mathrm{inter}}) =  \Theta_{\mathbb{P}}\left(d^{-(l+1)s} + d^{(2-\tilde{s})l-2\gamma} \right),
\end{displaymath}
where $\Theta_{\mathbb{P}} $ only involves constants depending on $s, \gamma, R_{\gamma}$, $c_{0}, c_{1}$ and $ c_{2}$.
 
\end{theorem}

Simple calculation shows that, when $d^{(2-\tilde{s})l-2\gamma}$ is the dominant term in $\mathrm{bias}^{2}(\hat{f}_{\mathrm{inter}})$, we always have $ d^{(2-\tilde{s})l-2\gamma} \ll d^{l-\gamma} + d^{\gamma-l-1}$. When $\gamma \in \mathbb{N}^{+}$, Theorem \ref{theorem variance} shows that $ \mathrm{var}(\hat{f}_{\mathrm{inter}}) = \Theta_{\mathbb{P}}(1)$. Therefore, we have the following corollary of Theorem \ref{theorem variance} and Theorem \ref{theorem bias}.

\begin{corollary}\label{corollary gen error}
    Let $c_{1} d^{\gamma} \le n \le c_{2} d^{\gamma} $ for some fixed $ \gamma >0 $ and absolute constants $ c_{1}, c_{2}$. Under the same notations and assumptions in Theorem \ref{theorem variance} and Theorem \ref{theorem bias}, the generalization error of kernel interpolation estimator $ \hat{f}_{\mathrm{inter}} $ satisfies
    \begin{equation}\label{eq rate of gen error}
        E_{x, \epsilon} \left[ \left( \hat{f}_{\mathrm{inter}}(x) - f^{*}(x) \right)^{2} \right] = \Theta_{\mathbb{P}} \left( d^{l-\gamma} + d^{\gamma-l-1} + d^{-(l+1)s} \right),
    \end{equation}
    where $\Theta_{\mathbb{P}} $ only involves constants depending on $\sigma^{2}, s, \gamma, R_{\gamma}$, $c_{0}, c_{1}$ and $ c_{2}$.
\end{corollary}

Here are some implications of Corollary \ref{corollary gen error}. Firstly, kernel interpolation is consistent when $s>0$ and $ \gamma \in (0,\infty) \backslash \mathbb{N}^{+}$; and is inconsistent when $s=0$ or $ \gamma \in \mathbb{N}^{+} $. Secondly, comparing the convergence rate in Eq.\eqref{eq rate of gen error} with the minimax lower rate in Proposition \ref{proposition minimax lower} (displayed in the Section \ref{subsection minimax} and borrowed from \cite{zhang2024optimal}), we know when kernel interpolation is optimal or sub-optimal. Specifically, for any fixed $\gamma > 0$, we define the threshold $\Gamma(\gamma)$ as follows:
\begin{align}\label{eqn:Gamma-gamma}
    \Gamma(\gamma) = \begin{cases} \infty,  & \gamma \in (0, 0.5], \\ 1-\gamma , & \gamma \in (0.5, 1.0], \\  (\gamma-l)/l, & \gamma \in (l, l+0.5] ~\text{for some}~ l \in \mathbb{N}^{+}, \\ (l+1-\gamma)/(l+1), & \gamma \in (l+0.5, l+1] ~\text{for some}~ l \in \mathbb{N}^{+}. \end{cases}
\end{align}
Then detailed elementary computation shows that when $s > \Gamma(\gamma)$, kernel interpolation is sub-optimal; and when $0 < s \le \Gamma(\gamma)$, kernel interpolation is minimax optimal. Intuitively, this indicates that when the true function is smoother (exceeding a certain threshold), kernel interpolation is sub-optimal.  
When $ \gamma \in (0, 0.5]$, kernel interpolation is optimal for any $s>0$ (the threshold $\Gamma(\gamma) = \infty$). 
 When $ \gamma > 0.5$, the threshold $\Gamma(\gamma)$ is no larger than $ 0.5$ and tends to 0 as $\gamma $ tends to $ \infty$. 
 $\Gamma(\gamma)$ is not continuous at $ \gamma = l+0.5, l \in \mathbb{N}$.
 The phase diagram ( Fig.\ref{fig1} ) provides a visualization of the minimax optimal, sub-optimal and inconsistent regions for different values of $\gamma$ and $s$.

\begin{figure}[htbp]
\vspace{-8pt}
\setlength{\abovecaptionskip}{-1pt}   
\setlength{\belowcaptionskip}{-2pt}   
\centering
\includegraphics[width=0.6\columnwidth]{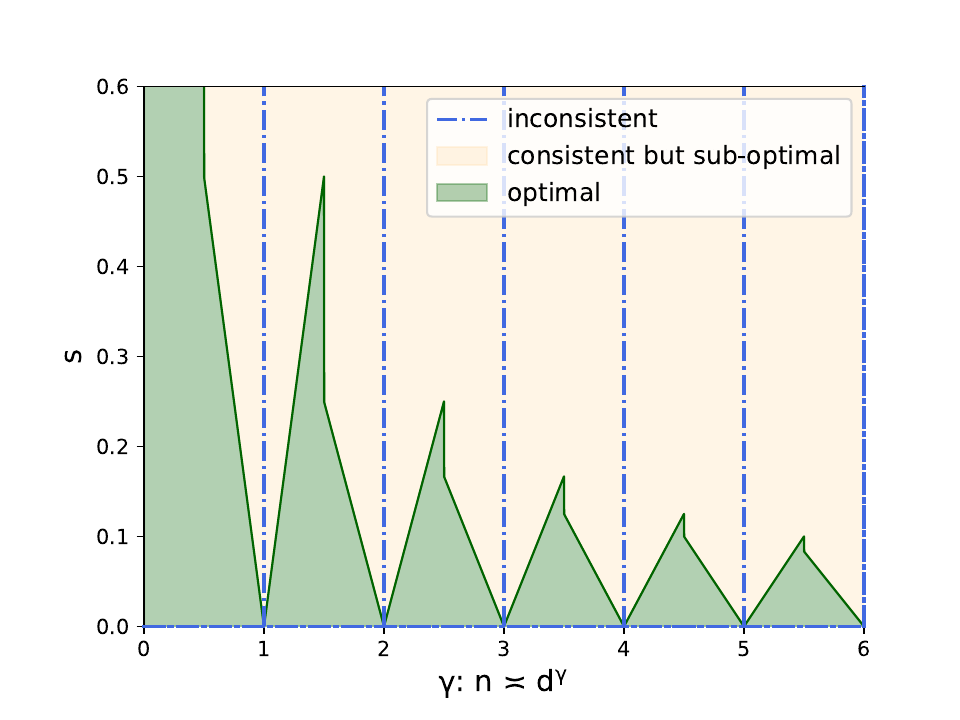}
\caption{An illustration of the optimal, sub-optimal and inconsistent regions. When $(s,\gamma)$ falls in the region below the curve (including the boundary), kernel interpolation is optimal; When $(s,\gamma)$ falls in the region above the curve (not including the boundary), kernel interpolation is sub-optimal; When $s=0$ or $\gamma \in \mathbb{N}^{+}$, kernel interpolation is inconsistent.}
\label{fig1}
\vspace{-8pt}
\end{figure}

\section{Discussion}\label{section discussion}

\subsection{Comparison with related works}\label{subsection related work}
There are a lot of excellent recent papers about kernel interpolation in large dimensions. Most of them only provided an upper bound of the generalization ability of the kernel interpolation. Contrast to these existing work, we obtained the exact convergence rates of kernel interpolation with different source condition (which is used to characterize the smoothness of $f^{*}$ ) and determined the regions where the kernel interpolation is optimal, sub-optimal and inconsistent.

\cite{liang2020multiple} studied the inner product kernel and assumed the coordinates of the $d$-dimensional input are independent and identical. They first proved an upper bound of the variance term, which is consistent with Theorem \ref{theorem variance}. Further assuming $ f^{*}(x) = \langle k(x,\cdot), \rho(\cdot)\rangle_{L^{2}}$, with $ \| \rho \|_{L^{4}}^{4} \le C$ for some constant $ C >0$, they demonstrated that the bias term is an infinitesimal of higher order compared with variance term. 


Another line of work \cite[etc.]{Ghorbani2019LinearizedTN,Donhauser_how_2021,xiao2022precise,hu2022sharp,misiakiewicz_spectrum_2022} studied the `polynomial approximation barrier', i.e., kernel regression is consistent if and only if $f^{*}$ is a polynomial with a fixed degree $\le \gamma$ (recall notation: $n \asymp d^{\gamma}$). They showed that the generalization error is approximately $ \| \mathrm{P}_{>l} f^{*} \|_{L^{2}} + o_{d,\mathbb{P}}(1)$ as $n,d \to \infty$, where $ \mathrm{P}_{\le l} $ denotes the orthogonal projection onto the subspace of polynomials of degree at most $l$ and $\mathrm{P}_{>l} = \mathrm{I} -\mathrm{P}_{\le l} $. When the source condition in Assumption \ref{assumption source condition} satisfies $s=0$, we have $ \| \mathrm{P}_{>l} f^{*} \|_{L^{2}} = \Theta(1) $ and our rate in Corollary \ref{corollary gen error} recovers the polynomial approximation barrier. For any $s>0$, we have $ \| \mathrm{P}_{>l} f^{*} \|_{L^{2}} = o(1) $ (see Lemma \ref{lemma cal of bias 2} in Section \ref{subsection auxiliary}). Thus, their results are not enough yet if we consider a slightly smaller function space $ [\mathcal{H}]^{s}, s>0$, rather than $ [\mathcal{H}]^{0}$ (by definition, $ [\mathcal{H}]^{0} \cong L^{2}$ for the inner product kernel in Assumption \ref{assumption inner product kernel}).

To our knowledge, \cite{aerni2022strong} is the only paper except us that provides lower bounds of the variance and bias for kernel interpolation in large dimensions. They considered a specific type of convolutional kernel on the discrete hypercube and a special form of the true function $f^{*}(x) = x_{1}x_{2}\cdots x_{L^{*}}$, where $ L^{*}$ is formulated in Theorem 1 in \cite{aerni2022strong}. Under this assumption, $f^{*}$ will fall into the span of first $m$ eigenfunctions, where $m \ll n$ was introduced in Lemma 2 in \cite{aerni2022strong}. By comparison, our assumption on $f^{*}$ does not have this restriction. As will be seen in our proof, we handle both the $< n$ and $>n$ components, which is also why the convergence rate of bias in Theorem \ref{theorem bias} contains two terms. 

A related work \cite{barzilai2023generalization} extended the upper bound of the generalization error of linear regression \cite{bartlett2020benign,tsigler2023benign} to kernel regression under mild assumptions. They also applied their results to kernel interpolation with the inner product kernel on the sphere. It is usually difficult to derive the matching lower bound (not provided in \cite{barzilai2023generalization}) of the convergence rates. For instance, Theorem 2 in \cite{tsigler2023benign} requires several particular assumptions to obtain the matching lower bound in ridge regression. Nevertheless, the lower bound is necessary to discuss the sub-optimality and inconsistency of kernel interpolation and we make it without assuming extra unrealistic assumptions.

\subsection{Minimax lower bound}\label{subsection minimax}
We have compared our convergence rate with the minimax lower rate at the end of Section \ref{subsection gen of ki}. Now, we formulate the minimax lower rate in the following proposition, which is borrowed from Theorem 5 in \cite{zhang2024optimal}. 

\begin{proposition}[Theorem 5 in \cite{zhang2024optimal}]
\label{proposition minimax lower}
Let $c_{1} d^{\gamma} \le n \le c_{2} d^{\gamma} $ for some fixed $ \gamma \in (0,\infty) \backslash \mathbb{N}^{+}$ and absolute constants $ c_{1}, c_{2}$. Consider model Eq.\eqref{main data model} and let $k = k_{d}$ be the inner product kernel on the sphere satisfying Assumption \ref{assumption inner product kernel}. For any $s>0$, we have:

        \begin{itemize}[leftmargin = 15pt]
            \item[(i)] When $\gamma \in \left(p+ps, p+ps+s \right]$ for some $p \in \mathbb{N}$, for any $\delta > 0$, there exist constants $C_1$ and $C$ only depending on $s, \delta, \gamma, \sigma, c_{1}$ and $ c_{2} $ such that for any $d \geq C$, we have:
            \begin{displaymath}
            \min _{\hat{f}} \max _{f^{*} \in [\mathcal{H}]^{s}} E_{X,Y}
            \left\|\hat{f} - f^{*}\right\|_{L^2}^2
            \geq C_1 d^{-\gamma + p - \delta};
            \end{displaymath}

            \item[(ii)] When $\gamma \in \left(p+ps+s, (p+1)+(p+1)s \right]$ for some $p \in \mathbb{N}$, there exist constants $C_1$ and $C$ only depending on $ s, \gamma, \sigma, c_{1}$ and $ c_{2} $ such that for any $d \geq C$, we have:
            \begin{displaymath}
            \min _{\hat{f}} \max _{f^{*} \in [\mathcal{H}]^{s}} E_{X,Y}
            \left\|\hat{f} - f^{*}\right\|_{L^2}^2
            \geq C_1 d^{-(p+1)s},
            \end{displaymath}
        
        \end{itemize}
    where $ E_{X,Y}$ means taking expectation on the train samples $\{ (x_{i}, y_{i}) \}_{i=1}^{n}$.
\end{proposition}


\section{Proof}
\subsection{More notations}\label{subsection more notiations}

In this subsection, we introduce the necessary notations for our proof. Denote the $ n $ samples as $X = \left( x_{1}, \ldots, x_{n} \right)^{\top}$ and $Y = (y_{1}, \ldots, y_{n})^{\top}$. Denote the random noise vector as $ \epsilon = \left( \epsilon_{1}, \ldots, \epsilon_{n} \right)^{\top}$.
We use $ \left\| \cdot \right\|_{\mathrm{op}} $ to denote the operator norm of a matrix and use $ \left\| \cdot \right\|_{2} $ to denote the $\ell^{2}$ norm of a vector. Recall the inner product kernel on the sphere (as in Assumption \ref{assumption inner product kernel}) given in the basis of spherical harmonics:
\begin{equation}\label{mercer of inner}
  k(x,x^{\prime}) = \sum_{k=0}^{\infty} \mu_k \sum_{m=1}^{N(d,k)} \psi_{k,m}(x) \psi_{k,m}(x^{\prime}),
\end{equation}
where $ \{\psi_{k,m}\}_{m=1}^{N(d,k)}$ are spherical harmonic polynomials of degree $k$; $ \mu_{k}$ are the eigenvalues with multiplicity $N(d,0)=1$; $N(d, k) = (2k+d-1) \cdot (k+d-2)! / \left[ k (d-1)!(k-1)! \right], k =1,2,\ldots$; and $ \{\psi_{k,m}\}_{k \ge 0, m \ge 1}$ is an orthonormal basis of $L^{2}$ space. In the proof, we treat the function $\Phi$ in Assumption \ref{assumption inner product kernel} as fixed and ignore the constants' dependence on $ \{ a_{j} \}_{j=0}^{\infty}$.

Recall that we consider the large-dimensional setting $n \asymp d^{\gamma}$ for some $\gamma >0$. Throughout the proof, we use the notation $l = \lfloor \gamma \rfloor$ and define
\begin{displaymath}
    B_{l} = B_{l}(d) = \sum_{k=0}^{l} N(d,k).
\end{displaymath}
Then Lemma \ref{lemma Ndk} will show that (the following $ \Theta $ only involves constants depending on $ \gamma $):
\begin{equation}\label{eq Bl asymp}
    B_{l} = \Theta(d^{l}).
\end{equation}
Next, we define two important quantities which will be frequently used in the proof:
\begin{equation}\label{eq def of kappa1 2}
    \kappa_{1} = \sum_{k=l+1}^{\infty} \mu_{k} N(d,k), \quad \kappa_{2} = \sum_{k=l+1}^{\infty} \mu_{k}^{2} N(d,k). 
\end{equation}

Denote $\psi(x) = (\psi_{k,m}(x))^{\top}_{k \ge 0, m \ge 1} \in \mathbb{R}^{\infty}$ and $ \Psi_{k,m} = \psi_{k,m}(X) = \left(\psi_{k,m}(x_{1}), \ldots \psi_{k,m}(x_{n}) \right)^{\top} \in \mathbb{R}^{n}$. For any $ k = 0, 1, \ldots$, denote 
\begin{displaymath}
    \Psi_{k} = \left(\Psi_{k,1}, \ldots, \Psi_{k,N(d,k)} \right) \in \mathbb{R}^{n \times N(d,k)},
\end{displaymath}
and $ \Psi = \left( \Psi_{\le l}, \Psi_{> l} \right) \in \mathbb{R}^{n \times \infty}$, where
\begin{displaymath}
    \Psi_{\le l} = \left(\Psi_{1}, \ldots, \Psi_{l} \right) \in \mathbb{R}^{n \times B_{l}}, \quad \Psi_{> l} = \left(\Psi_{l+1}, \Psi_{l+2}, \ldots \right) \in \mathbb{R}^{n \times \infty}.
\end{displaymath}
For any $ k = 0, 1, \ldots$, define the diagonal matrix (~$\mathrm{I}_{N(d,k)} $ is the identity matrix) 
\begin{displaymath}
    \Sigma_{k} = \mu_{k} \mathrm{I}_{N(d,k)} \in \mathbb{R}^{N(d,k) \times N(d,k)},
\end{displaymath}
and $ \Sigma = \text{Diag}  \left\{ \Sigma_{\le l}, \Sigma_{> l} \right\} \in \mathbb{R}^{\infty \times \infty}$, where
\begin{displaymath}
\Sigma_{\le l} = \text{Diag} \left\{\Sigma_{0}, \ldots, \Sigma_{l} \right\} \in \mathbb{R}^{B_{l} \times B_{l}}, \quad \Sigma_{> l} = \text{Diag} \left\{\Sigma_{l+1}, \Sigma_{l+2}, \ldots \right\} \in \mathbb{R}^{\infty \times \infty}.
\end{displaymath}
In Assumption \ref{assumption source condition}, we denote the true function as $ f^{*}(x) = \sum_{k=0}^{\infty} \sum_{m=1}^{N(d,k)} \theta^{*}_{k,m} \psi_{k,m}(x)$. Further denote the coefficient vectors 
\begin{displaymath}
    \theta^{*}_{k} = \left( \theta_{k,1}^{*}, \ldots, \theta_{k,N(d,k)}^{*} \right)^{\top} \in \mathbb{R}^{N(d,k)},
\end{displaymath}
and $ \theta^{*} = \left( \theta^{* \top}_{\le l}, \theta^{* \top}_{> l} \right)^{\top} \in \mathbb{R}^{\infty}$, where
\begin{displaymath}
    \theta^{*}_{\le l} = \left( \theta_{1}^{* \top}, \ldots, \theta_{l}^{* \top} \right)^{\top} \in \mathbb{R}^{B_{l}}, \quad \theta^{*}_{> l} = \left( \theta_{l+1}^{* \top}, \theta_{l+2}^{* \top}, \ldots \right)^{\top} \in \mathbb{R}^{\infty}.
\end{displaymath}
Now we have the following expressions:
\begin{displaymath}
    f^{*}(x) = \psi(x) \theta^{*}, \quad f^{*}(X) = \left( f^{*}(x_{1}), \ldots, f^{*}(x_{n}) \right)^{\top} = \Psi \theta^{*}.
\end{displaymath}
Denote the sample kernel matrix $ K = k(X,X) \in \mathbb{R}^{n \times n}$, where $ K_{ij} = k(x_{i},x_{j})$. Further define $K = K_{\le l} + K_{> l} $, where $ K_{\le l} = \Psi_{\le l} \Sigma_{\le l} \Psi_{\le l}^{\top} $, i.e.,  
\begin{displaymath}
    (K_{\le l})_{i,j} = \sum\limits_{k=0}^{l} \mu_{k} \sum\limits_{m=1}^{N(d,k)} \psi_{k,m}(x_{i}) \psi_{k,m}(x_{j});
\end{displaymath}
and $ K_{> l} = \Psi_{>l} \Sigma_{> l} \Psi_{> l}^{\top} $, i.e.,  
\begin{displaymath}
    (K_{> l})_{i,j} = \sum\limits_{k=l+1}^{\infty} \mu_{k} \sum\limits_{m=1}^{N(d,k)} \psi_{k,m}(x_{i}) \psi_{k,m}(x_{j}). 
\end{displaymath}

\subsection{Eigenvalues and eigenfunctions of the inner product kernel}
\label{subsection eigen of inner property}

In this subsection, we state several useful lemmas about the eigenvalues and eigenfunctions of the inner product kernel, which will be the main workhorse of our proof.

The following Lemma \ref{lemma inner eigen}, \ref{lemma:monotone_of_eigenvalues_of_inner_product_kernels} and \ref{lemma Ndk} (cited from Lemma 3.1, 3.3 and B.1 in \cite{lu2023optimal}) give concise characterizations of $\mu_{k}$ and $N(d,k)$, which is sufficient for the analysis in this paper.

\begin{lemma}\label{lemma inner eigen}
    Let $k = k_{d}$ be the inner product kernel on the sphere satisfying Assumption \ref{assumption inner product kernel}. For any fixed integer $p \ge 0$, there exist constants $C, C_1$ and $C_2$ only depending on $p$ and $\{a_j\}_{j \leq p+1}$, such that for any $d \geq C$, we have
\begin{equation}
\begin{aligned}
{C_1}{d^{-k}} &\leq \mu_{k} \leq {C_2}{d^{-k}}, ~~ k=0,1,\ldots, p+1. \notag
\end{aligned}
\end{equation}
\end{lemma}

\begin{lemma}\label{lemma:monotone_of_eigenvalues_of_inner_product_kernels}
    Let $k = k_{d}$ be the inner product kernel on the sphere satisfying Assumption \ref{assumption inner product kernel}. For any fixed integer $p \ge 0$, there exist constants $C$ only depending on $p$ and $\{a_j\}_{j \leq p+1}$, such that for any $d \geq C$, we have
    \begin{equation*}
        \mu_k \leq \frac{C_2}{C_1} d^{-1} \mu_{p}, \quad k=p+1, p+2, \cdots
    \end{equation*}
    where $C_1$ and $C_2$ are constants given in Lemma \ref{lemma inner eigen}.
\end{lemma}
Lemma \ref{lemma:monotone_of_eigenvalues_of_inner_product_kernels} indicates that when we consider $\mu_{p}$ with any fixed $ p \ge 0$, the subsequent eigenvalues $ \mu_{k}$ ($ k\ge p+1$) are much smaller than $\mu_{p}$. 

\begin{lemma}\label{lemma Ndk}
    For an integer $ k \ge 0$, denote $N(d,k)$ as the multiplicity of the eigenspace corresponding to $\mu_{k}$ in Eq.\eqref{mercer of inner}. For any fixed integer $p \ge 0$, there exist constants $C_3, C_4$ and $C$ only depending on $p$, such that for any $d \ge C$, we have
    \begin{displaymath}
        C_3 d^k \le N(d, k)  \le C_4 d^k, \quad k = 0, 1, \cdots, p+1.
    \end{displaymath}
\end{lemma}

The following lemma determines the exact order of $\kappa_{1}, \kappa_{2}$.
\begin{lemma}\label{lemma rate of k1 k2}
For $\kappa_{1}, \kappa_{2}$ defined in Eq.\eqref{eq def of kappa1 2}, we have 
\begin{align}
    \kappa_{1} = \Theta(1), \quad \kappa_{2} = \Theta\left( d^{-(l+1)} \right), \notag
\end{align}
where $\Theta$ only involves constants depending on $\gamma$.
\end{lemma}

\begin{proof}
    In this proof, all the constants only depend on $\gamma$ (or equivalently, $l$). We first prove $\kappa_{1} = \Theta(1)$. Since we always assume that the integral operator $L_{k}$ in Section \ref{section interpolation space and rkhs} satisfies $ \| L_{k} \|_{\mathrm{TR}} \le 1 $. By the definition of trace norm, we have $ \kappa_{1} \le \sum_{i=1}^{\infty} \lambda_{i} = \| L_{k} \|_{\mathrm{TR}} \le 1$. Lemma \ref{lemma inner eigen}, \ref{lemma Ndk} show that $ \kappa_{1} \ge \mu_{l+1} N(d,l+1) \gtrsim d^{-(l+1)} \cdot d^{l+1} = 1$.

    Next we prove $\kappa_{2} = \Theta\left( d^{-(l+1)} \right) $. On the one hand, Lemma \ref{lemma:monotone_of_eigenvalues_of_inner_product_kernels} implies that
    \begin{align}
        \kappa_{2} &\lesssim \mu_{l+1} \sum\limits_{k=l+1} \mu_{k} N(d,k). \notag
    \end{align}
    Using $\kappa_{1} = \Theta(1)$ and the fact that $ \mu_{l+1} \lesssim d^{-(l+1)}$ by Lemma \ref{lemma inner eigen}, we prove that $ \kappa_{2} \lesssim d^{-(l+1)} $. On the other hand, using Lemma \ref{lemma inner eigen}, \ref{lemma Ndk} again, we have
    \begin{displaymath}
        \kappa_{2} \ge \mu_{l+1}^{2} N(d,l+1) \gtrsim d^{-2(l+1)} \cdot d^{l+1} = d^{-(l+1)}.
    \end{displaymath}
    We finish the proof.
\end{proof}

The following lemma is borrowed from Proposition 3 and Eq.(56) in \cite{Ghorbani2019LinearizedTN}. We state this lemma without proof.
\begin{lemma}\label{lemma psi psi top}
   Let $c_{1} d^{\gamma} \le n \le c_{2} d^{\gamma} $ for some fixed $ \gamma \in (0,\infty)\backslash\mathbb{N}^{+}$ and absolute constants $ c_{1}, c_{2}$. Recall that $l = \lfloor \gamma \rfloor$, then we have
    \begin{displaymath}
        \sup\limits_{p \ge l+1} \left\| \frac{\Psi_{p} \Psi_{p}^{\top} }{N(d,p)} - \mathrm{I}_{n} \right\|_{\mathrm{op}} = o_{\mathbb{P}}(1),
    \end{displaymath}
    which further implies
    \begin{align}\label{eq k ge l in lemma}
        K_{> l} = \Psi_{>l} \Sigma_{>l} \Psi_{>l}^{\top} = \kappa_{1} \left(\mathrm{I}_{n} + \Delta_{1}\right), \quad \|\Delta_{1}\|_{\mathrm{op}} = o_{\mathbb{P}}(1),
    \end{align}
    and
    \begin{equation}\label{eq k2 ge l in lemma}
        \Psi_{>l} \Sigma_{>l}^{2} \Psi_{>l}^{\top} = \kappa_{2} \left(\mathrm{I}_{n} + \Delta_{2}\right), \quad \|\Delta_{2}\|_{\mathrm{op}} = o_{\mathbb{P}}(1),
    \end{equation}
     where $o_{\mathbb{P}}$ only involves constants depending on $\gamma, c_{1}$ and $c_{2}$.
\end{lemma}

Throughout the proof, we use $\Delta_{1}, \Delta_{2}$ specifically for the quantities that appeared in Lemma \ref{lemma psi psi top}.

The following lemma is borrowed from Lemma 11 in \cite{Ghorbani2019LinearizedTN}. We state this lemma without proof.
\begin{lemma}\label{lemma psi top psi}
Let $c_{1} d^{\gamma} \le n \le c_{2} d^{\gamma} $ for some fixed $ \gamma \in (0,\infty) \backslash \mathbb{N}^{+}$ and absolute constants $ c_{1}, c_{2}$. Recall that $l = \lfloor \gamma \rfloor$. For any $p \in \mathbb{N}$ and $ p \le l$, denote $B_{p} = \sum_{k=0}^{p} N(d,k)$, then we have
    \begin{displaymath}
         \left\| \frac{\Psi_{\le p}^{\top} \Psi_{\le p} }{n} - \mathrm{I}_{B_{p}} \right\|_{\mathrm{op}} = o_{\mathbb{P}}(1),
    \end{displaymath}
    where $o_{\mathbb{P}}$ only involves constants depending on $\gamma, c_{1}$ and $c_{2}$.
\end{lemma}

\subsection{Main proofs of the theorems}

The representer theorem (see, e.g., \cite{steinwart2008support}) gives an explicit expression of the kernel ridge regression estimator \eqref{eq def krr}, i.e., 
\begin{displaymath}
   \hat{f}_{\lambda}(x) = k(x, X)(K+n \lambda \mathrm{I}_{n})^{-1} Y,
\end{displaymath}
where $ K = k(X,X) \in \mathbb{R}^{n \times n}$ is the sample kernel matrix and $ k(x, X)=\left(k\left(x, x_1\right), \cdots, k\left(x, x_n\right)\right)$. Then the kernel interpolation estimator \eqref{eq:KMinNorm} can be expressed as 
\begin{displaymath}
    \hat{f}_{\mathrm{inter}} = k(x,X) K^{-1} Y.
\end{displaymath}

\begin{proof}[of Theorem~\ref{theorem variance}]

In the case of $\gamma \in \mathbb{N}^{+}$, \cite{xiao2022precise,misiakiewicz_spectrum_2022,hu2022sharp} have already proved the precise estimates of $\mathrm{var}(\hat{f}_{\mathrm{inter}})$. Their results directly implies that $ \mathrm{var}(\hat{f}_{\mathrm{inter}}) = \Theta_{\mathbb{P}}(1)$, which is exactly what we stated in Theorem \ref{theorem variance} when $ \gamma \in \mathbb{N}^{+} $ (at this time, we have $ l = \lfloor \gamma \rfloor = \gamma$). Since their proof requires much extra knowledge beyond this paper, we refer to, e.g., Theorem 2 in \cite{xiao2022precise} for details. In the following of this proof, we safely focus on the case of $\gamma \in (0,\infty) \backslash \mathbb{N}^{+}$. All the constants that appear in the proof of this theorem only depend on $s, \gamma$, $c_{1}$ and $ c_{2}$. 

Firstly, the variance term Eq.\eqref{eq var term} can be expressed as
\begin{align}\label{eq var tr}
    \mathrm{var}\left( \hat{f}_{\mathrm{inter}} \right) &= E_{x, \epsilon}\left[\left(k(x,X) K^{-1} \epsilon \right)^2\right] = E_{x,\epsilon} \left[\left( \psi(x) \Sigma^{1/2} \Sigma^{1/2} \Psi^{\top} K^{-1} \epsilon \right)^2\right] \notag \\
    &=  \sigma^{2} \mathrm{tr} \left( K^{-1} \Psi \Sigma^{2} \Psi^{\top} K^{-1} \right).
\end{align}
Using Eq.\eqref{eq k2 ge l in lemma} in Lemma \ref{lemma psi psi top} and the fact that $ \Sigma $ is a diagonal matrix, we can decompose Eq.\eqref{eq var tr} as
\begin{equation}\label{eq var tr decomposition}
    \mathrm{var}\left( \hat{f}_{\mathrm{inter}} \right) / \sigma^{2} = \mathrm{tr} \underbrace{\left(\kappa_{2}  K^{-1} \left(\mathrm{I}_{n} + \Delta_{2}\right) K^{-1}\right)}_{ V_{1}}  ~+~ \mathrm{tr} \underbrace{\left( K^{-1} \Psi_{\le l} \Sigma_{\le l}^{2} \Psi_{\le l}^{\top} K^{-1}\right)}_{ V_{2}},
\end{equation}
where $ \kappa_{2}$ is defined in Eq.\eqref{eq def of kappa1 2} and $ \Delta_{2}$ is introduced in Lemma \ref{lemma psi psi top}.

For $V_{1}$ in Eq.\eqref{eq var tr decomposition}, using Weyl's inequality recalling that $\|\Delta_{2}\|_{\mathrm{op}} = o_{\mathbb{P}}(1)$ , we have
\begin{align}
        \lambda_{\text{max}}\left( V_{1} \right) &= \kappa_{2} \cdot \lambda_{\text{max}}\left( K^{-1} \left( \mathrm{I}_{n} + \Delta_{2} \right)  K^{-1}\right) \notag \\
        &\overset{\text{Weyl's}}{\leq} \kappa_{2} \cdot \left[ \lambda_{\text{max}}\left( K^{-1} K^{-1}\right) + \lambda_{\text{max}}\left( K^{-1} \Delta_{2} K^{-1}\right)  \right] \notag \\
        &\le \kappa_{2} \cdot \lambda_{\text{min}}\left( K\right)^{-2} \left( 1 + o_{\mathbb{P}}(1) \right). \notag
    \end{align}
    Recalling that $\kappa_{1} = \Theta(1)$ in Lemma \ref{lemma rate of k1 k2} and $\|\Delta_{1}\|_{\mathrm{op}} = o_{\mathbb{P}}(1)$ in Lemma \ref{lemma psi psi top}, we have
    \begin{align}
        \lambda_{\text{min}}\left( K\right) &\overset{\text{Weyl's}}{\geq} \lambda_{\text{min}} \left( \Psi_{\le l} \Sigma_{\le l} \Psi_{\le l}^{\top} \right) + \lambda_{\text{min}} \left( \kappa_{1} \mathrm{I}_{n} \right) + \lambda_{\text{min}} \left( \kappa_{1} \Delta_{1} \right) \notag \\
        &\ge \lambda_{\text{min}} \left( \kappa_{1} \mathrm{I}_{n} \right) + \lambda_{\text{min}} \left( \kappa_{1} \Delta_{1} \right) \notag \\
        &= \Theta_{\mathbb{P}}(1). \notag
    \end{align}
    We have proved in Lemma \ref{lemma rate of k1 k2} that $ \kappa_{2} = \Theta(d^{-(l+1)})$. Therefore, we have $\lambda_{\text{max}}\left( V_{1} \right) = O_{\mathbb{P}} \left(  d^{-(l+1)} \right)$, which further implies that
    \begin{equation}\label{eq tr v1 up}
       \mathrm{tr} \left( V_{1} \right)  = O_{\mathbb{P}} \left( \frac{n}{d^{(l+1)} } \right).
    \end{equation}
    On the other hand, we have
    \begin{align}
        \mathrm{tr}\left( V_{1} \right) &=  \kappa_{2} \cdot \mathrm{tr}\left( K^{-1} \left( \mathrm{I}_{n} + \Delta_{2} \right)  K^{-1}\right) \ge  \kappa_{2} \cdot \mathrm{tr}\left( K^{-1} K^{-1}\right) \left( 1 - o_{\mathbb{P}}(1) \right) \notag \\
        &\ge  \kappa_{2} \cdot n \left[ \mathrm{tr}\left( \frac{K}{n} \right) \right]^{-2} \left( 1 -o_{\mathbb{P}}(1) \right), \notag
    \end{align}
    where we use Jensen inequality and the fact that the eigenvalues of $K$ are positive for the last inequality. Further, we have
    \begin{align}
        \mathrm{tr}\left( \frac{K}{n} \right) &= \mathrm{tr}\left( \frac{1}{n} \Psi_{\le l} \Sigma_{\le l} \Psi_{\le l}^{\top} \right) + \mathrm{tr}\left( \frac{\kappa_{1}}{n} \mathrm{I}_{n}  \right) + \mathrm{tr}\left( \frac{\kappa_{1}}{n} \mathbf{\Delta}_{1}  \right) \notag \\
        &= \mathrm{tr}\left( \frac{1}{n} \Psi_{\le l}^{\top} \Psi_{\le l} \Sigma_{\le l} \right) + \kappa_{1} \left( 1 + o_{\mathbb{P}}(1)\right) \notag \\
        &= \mathrm{tr}\left( \mathrm{I}_{B_{l}} \Sigma_{\le l}) (1 + o_{\mathbb{P}}(1)\right) + \kappa_{1} \left( 1 + o_{\mathbb{P}}(1)\right) \notag \\
        &= O_{\mathbb{P}}(1), \notag
    \end{align}
    where we use Lemma \ref{lemma psi top psi} for the third equality. Therefore, we can prove that 
    \begin{equation}\label{eq tr v1 low}
       \mathrm{tr}\left( V_{1} \right) = \Omega_{\mathbb{P}}\left( n \kappa_{2} \right) = \Omega_{\mathbb{P}}\left( \frac{n}{d^{l+1}}\right).
    \end{equation}
Combining Eq.\eqref{eq tr v1 up} and Eq.\eqref{eq tr v1 low}, we have
    \begin{equation}\label{eq tr v1 exact}
       \mathrm{tr}\left( V_{1} \right) = \Theta_{\mathbb{P}}\left( \frac{n}{d^{l+1}}\right).
    \end{equation}

For $V_{2}$ in Eq.\eqref{eq var tr decomposition}, using Lemma \ref{lemma trans in V2} for the second equality below, we have
\begin{align}
    \mathrm{tr} \left( V_{2} \right) &= \mathrm{tr} \left( \Sigma_{\le l} \Psi_{\le l}^{\top}  K^{-2} \Psi_{\le l} \Sigma_{\le l} \right) \notag \\
    &= \frac{1}{n} \mathrm{tr} \left[ \left( \frac{\Sigma_{\le l}^{-1}}{n} + \frac{\Psi_{\le l}^{\top} K_{>l}^{-1} \Psi_{\le l}}{n}   \right)^{-1} \frac{\Psi_{\le l}^{\top} K_{>l}^{-2} \Psi_{\le l}}{n} \left( \frac{\Sigma_{\le l}^{-1}}{n} + \frac{\Psi_{\le l}^{\top} K_{>l}^{-1} \Psi_{\le l}}{n}   \right)^{-1}  \right]. \notag
\end{align}
Using Lemma \ref{lemma lambda max min}, we have
\begin{align}\label{eq tr v2 exact}
    \mathrm{tr} \left( V_{2} \right) = \Theta_{\mathbb{P}} \left[ \frac{1}{n} \mathrm{tr} \left(  \frac{\Psi_{\le l}^{\top} K_{>l}^{-2} \Psi_{\le l}}{n} \right)\right] = \Theta_{\mathbb{P}}\left[  \frac{1}{n} \mathrm{tr} \left(  \frac{\Psi_{\le l}^{\top} \Psi_{\le l}}{n} \right)\right]  = \Theta_{\mathbb{P}}\left( \frac{d^{l}}{n}\right),
\end{align}
where we use Eq.\eqref{eq k ge l in lemma} in Lemma \ref{lemma psi psi top} and the fact that $ \kappa_{1} = \Theta(1)$ for the second equality; and we use Lemma \ref{lemma psi top psi} for the last equality. 

Finally, plugging Eq.\eqref{eq tr v1 exact}, \eqref{eq tr v2 exact} into Eq.\eqref{eq var tr decomposition}, we finish the proof of Theorem \ref{theorem variance}. 

$\hfill\square$

\end{proof}

\begin{proof}[of Theorem~\ref{theorem bias}]

All the constants that appear in the proof of this theorem only depend on $s, \gamma, R_{\gamma}$, $c_{0}, c_{1}$ and $ c_{2}$. Firstly, the bias term Eq.\eqref{eq bias term} can be expressed as
\begin{align}
    \mathrm{bias}^2(\hat{f}_{\mathrm{inter}}) &= E_{x} \left[\left( k(x,X) K^{-1} f^{*}(X) -f^{*}(x)\right)^2\right] \notag \\
    &= E_{x} \left[\left( \psi(x) \Sigma^{1/2} \Sigma^{1/2} \Psi^{\top} K^{-1} \Psi \theta^{*} - \psi(x) \theta^{*}\right)^2\right] \notag \\
    &= \| \Sigma \Psi^{\top} K^{-1} \Psi \theta^{*} - \theta^{*}  \|_{2}^{2}. \notag
\end{align}
Deriving the lower bounds of the bias term is the main obstacle in related work. We make it by making the following decomposition Eq.\eqref{eq decom of bias}, calculating the tight bounds of the two terms in Eq.\eqref{eq decom of bias}, and discussing two cases: $ 0 \le s \le 2\gamma - 2l$ and $ s > 2 \gamma - 2l$.  Since $ \Sigma $ is a diagonal matrix, we decompose the bias term into the $\le l$ components and the $> l$ components:
\begin{equation}\label{eq decom of bias}
    \mathrm{bias}^2(\hat{f}_{\mathrm{inter}}) = \underbrace{\| \Sigma_{\le l} \Psi^{\top}_{\le l} K^{-1} \Psi \theta^{*} - \theta^{*}_{\le l}  \|_{2}^{2}}_{B_{1}} + \underbrace{\| \Sigma_{> l} \Psi^{\top}_{> l} K^{-1} \Psi \theta^{*} - \theta^{*}_{> l}  \|_{2}^{2}}_{B_{2}}.
\end{equation} 
The key point of our proof is: 1) when $ 0 \le s \le 2\gamma - 2l$, the dominant term is $ \| \theta^{*}_{>l}  \|_{2}^{2} = \Theta\left( d^{-(l+1)s} \right) $ from $B_{2}$ ;~~ 2) when $ s > 2 \gamma - 2l$, the dominant term is $ \| \Sigma_{\le l}^{-1} \theta^{*}_{\le l}  \|_{2}^{2} / n^{2} =  \Theta\left( d^{(2-\tilde{s})l - 2\gamma} \right) $ from $ B_{1}$ .   \\

Case 1: $ 0 \le s \le 2\gamma - 2l$. In this case, we have $ d^{(2-\tilde{s})l-2 \gamma} \lesssim d^{-(l+1)s} $.

For $B_{1}$ in Eq.\eqref{eq decom of bias}, using the fact that $\Psi \theta^{*} = \Psi_{\le l} \theta^{*}_{\le l} + \Psi_{> l} \theta^{*}_{> l} $, we have 
\begin{align}\label{eq decomposition B1}
    B_{1} &= \left\| \Sigma_{\le l} \Psi^{\top}_{\le l} K^{-1} \Psi_{\le l} \theta^{*}_{\le l} + \Sigma_{\le l} \Psi^{\top}_{\le l} K^{-1} \Psi_{> l} \theta^{*}_{> l} - \theta^{*}_{\le l}  \right\|_{2}^{2} \notag \\
    &\le 2 \underbrace{\left\| \Sigma_{\le l} \Psi^{\top}_{\le l} K^{-1} \Psi_{\le l} \theta^{*}_{\le l} - \theta^{*}_{\le l}  \right\|_{2}^{2}}_{B_{1,1}} + 2 \underbrace{\left\|  \Sigma_{\le l} \Psi^{\top}_{\le l} K^{-1} \Psi_{> l} \theta^{*}_{> l} \right\|_{2}^{2}}_{B_{1,2}} .
\end{align}
Since we have $d^{(2-\tilde{s})l-2 \gamma} \lesssim d^{-(l+1)s}$, then the bounds of $B_{1,1}, B_{1,2}$ in Lemma \ref{lemma B11}, \ref{lemma B12} imply that: 
\begin{equation}\label{eq ref B1-1}
    B_{1} =  O_{\mathbb{P}}\left( d^{-(l+1)s} \right).
\end{equation}
For $ B_{2}$ in Eq.\eqref{eq decom of bias}, on the one hand, using $\Psi \theta^{*} = \Psi_{\le l} \theta^{*}_{\le l} + \Psi_{> l} \theta^{*}_{> l} $, we have 
    \begin{align}\label{eq decomposition B2}
        B_{2} &= \left\| \Sigma_{> l} \Psi^{\top}_{> l} K^{-1} \Psi_{\le l} \theta^{*}_{\le l} + \Sigma_{> l} \Psi^{\top}_{> l} K^{-1} \Psi_{> l} \theta^{*}_{> l} - \theta^{*}_{> l}  \right\|_{2}^{2} \notag \\
        &\le 3 \underbrace{ \|\theta^{*}_{>l} \|_{2}^{2}}_{B_{2,1}} +  3 \underbrace{ \left\| \Sigma_{> l} \Psi^{\top}_{> l} K^{-1} \Psi_{\le l} \theta^{*}_{\le l}  \right\|_{2}^{2}}_{B_{2,2}} + 3 \underbrace{ \left\|  \Sigma_{>l} \Psi^{\top}_{> l} K^{-1} \Psi_{> l} \theta^{*}_{> l} \right\|_{2}^{2}}_{ B_{2,3}}.
    \end{align}
    We have already show in Lemma \ref{lemma cal of bias 2} that $ B_{2,1} = \Theta_{\mathbb{P}} \left( d^{-(l+1)s} \right)$. Using $d^{(2-\tilde{s})l-2 \gamma} \lesssim d^{-(l+1)s}$ again, Lemma \ref{lemma B22}, \ref{lemma B23} show that $ B_{2,2} = B_{2,3} = o_{\mathbb{P}}\left( d^{-(l+1)s} \right) $ , which further implies that: 
\begin{equation}\label{eq ref B2-1}
  \quad B_{2} = O_{\mathbb{P}}\left( d^{-(l+1)s} \right).
\end{equation}
On the other hand, we have
\begin{align}
    (B_{2})^{1/2} \ge \|\theta^{*}_{>l} \|_{2} -  \left\| \Sigma_{> l} \Psi^{\top}_{> l} K^{-1} \Psi_{\le l} \theta^{*}_{\le l}  \right\|_{2}-\left\|  \Sigma_{>l} \Psi^{\top}_{> l} K^{-1} \Psi_{> l} \theta^{*}_{> l} \right\|_{2}. \notag
\end{align}
Using Lemma \ref{lemma cal of bias 2}, \ref{lemma B22}, \ref{lemma B23} as above, we have
\begin{equation}\label{eq ref B2-2}
  \quad B_{2} = \Omega_{\mathbb{P}}\left( d^{-(l+1)s} \right).
\end{equation}
Plugging Eq.\eqref{eq ref B1-1}, \eqref{eq ref B2-1} and \eqref{eq ref B2-2} into Eq.\eqref{eq decom of bias}, we finally prove that when $ 0 \le s \le 2\gamma - 2l$, we have
\begin{displaymath}
    \mathrm{bias}^2(\hat{f}_{\mathrm{inter}}) = \Theta_{\mathbb{P}}\left( d^{-(l+1)s} \right).
\end{displaymath}

Case 2: $ s > 2\gamma - 2l$. In this case, we have $  d^{-(l+1)s} \ll d^{(2-\tilde{s})l-2 \gamma} $.

For $B_{1}$ in Eq.\eqref{eq decom of bias}, on the one hand, using Lemma \ref{lemma B11}, \ref{lemma B12} and the decomposition Eq.\eqref{eq decomposition B1}, we have
\begin{equation}\label{eq ref2 B1-1}
    B_{1} = O_{\mathbb{P}}\left( d^{(2-\tilde{s})l-2\gamma} \right).
\end{equation}
On the other hand, we have
\begin{align}\label{eq lower dec B1}
    (B_{1})^{1/2} \ge \left\| \Sigma_{\le l} \Psi^{\top}_{\le l} K^{-1} \Psi_{\le l} \theta^{*}_{\le l} - \theta^{*}_{\le l}  \right\|_{2} - \left\|  \Sigma_{\le l} \Psi^{\top}_{\le l} K^{-1} \Psi_{> l} \theta^{*}_{> l} \right\|_{2}.
\end{align}
Since $ d^{-(l+1)s} \ll d^{(2-\tilde{s})l-2 \gamma} $, Lemma \ref{lemma B11}, \ref{lemma B12} show that $ B_{1,1}=\Theta_{\mathbb{P}} \left( d^{(2-\tilde{s})l-2 \gamma} \right), B_{1,2}= o_{\mathbb{P}} \left( d^{(2-\tilde{s})l-2 \gamma} \right)$, which further imply
\begin{equation}\label{eq ref2 B1-2}
    B_{1} = \Omega_{\mathbb{P}}\left( d^{(2-\tilde{s})l-2\gamma} \right).
\end{equation}
For $B_{2}$ in Eq.\eqref{eq decom of bias}, using Lemma \ref{lemma cal of bias 2}, \ref{lemma B22}, \ref{lemma B23}, and the decomposition Eq.\eqref{eq decomposition B2}, we have
\begin{equation}\label{eq ref2 B2-1}
    \quad B_{2} = O_{\mathbb{P}}\left( d^{(2-\tilde{s})l-2\gamma} \right).
\end{equation}
Plugging Eq.\eqref{eq ref2 B1-1}, \eqref{eq ref2 B1-2} and \eqref{eq ref2 B2-1} into Eq.\eqref{eq decom of bias}, we finally prove that when $ s > 2\gamma - 2l$, we have
\begin{displaymath}
    \mathrm{bias}^2(\hat{f}_{\mathrm{inter}}) = \Theta_{\mathbb{P}}\left( d^{(2-\tilde{s})l-2\gamma} \right).
\end{displaymath}

To sum up, for any $ s \ge 0$, we prove that  
\begin{displaymath}
    \mathrm{bias}^2(\hat{f}_{\mathrm{inter}}) = \Theta_{\mathbb{P}}\left( d^{-(l+1)s} + d^{(2-\tilde{s})l-2\gamma} \right).
\end{displaymath}
We finish the proof of Theorem \ref{theorem bias}.

$\hfill\square$

\end{proof}

\subsection{Auxiliary results}\label{subsection auxiliary}

In this subsection, we provide some intermediate results that were used in the previous proof.

The following lemma was frequently used in the proof.
\begin{lemma}\label{lemma lambda max min}
Let $c_{1} d^{\gamma} \le n \le c_{2} d^{\gamma} $ for some fixed $ \gamma \in (0,\infty) \backslash \mathbb{N}^{+}$ and absolute constants $ c_{1}, c_{2}$. Recall that $l = \lfloor \gamma \rfloor$, then we have
    \begin{displaymath}
        \lambda_{\text{max}} \left[  \left( \frac{\Sigma_{\le l}^{-1}}{n} + \frac{\Psi_{\le l}^{\top} K_{>l}^{-1} \Psi_{\le l}}{n}   \right)^{-1} \right]  \asymp \lambda_{\text{min}} \left[  \left( \frac{\Sigma_{\le l}^{-1}}{n} + \frac{\Psi_{\le l}^{\top} K_{>l}^{-1} \Psi_{\le l}}{n}  \right)^{-1} \right] = \Theta_{\mathbb{P}}(1),
    \end{displaymath}
    where $\Theta_{\mathbb{P}}$ only involves constants depending on $\gamma, c_{1}$ and $c_{2}$.
\end{lemma}
\begin{proof}
    Lemma \ref{lemma inner eigen} implies that $ (\mu_{l})^{-1} \asymp d^{l} \ll n$ Lemma \ref{lemma rate of k1 k2} implies $ \kappa_{1} = \Theta(1)$. Therefore, using Eq.\eqref{eq k ge l in lemma} in Lemma \ref{lemma psi psi top} and Lemma \ref{lemma psi top psi}, we finish the proof.
\end{proof}

The following lemma is an application of the Markov's inequality, which was also used in \cite[Lemma 3]{barzilai2023generalization}.
\begin{lemma}\label{lemma ge l l2}
Recall that $l = \lfloor \gamma \rfloor$. With probability at least $ 1- \delta$, we have
    \begin{displaymath}
        \left\| \Psi_{> l} \theta_{>l}^{*} \right\|_{2}^{2} \le \frac{1}{\delta} n \| \theta_{>l}^{*} \|_{2}^{2}.
    \end{displaymath}
\end{lemma}
\begin{proof}
    Define $\beta_{i} = \left[\left(\Psi_{>l} \theta_{>l}^{*}\right)_{i}\right]^{2} = \left( \sum_{k=l+1}^{\infty} \sum_{m=1}^{N(d,k)} \psi_{k,m}(x_{i}) \theta_{k,m}^{*}\right)^{2} $. Then we have $\left\| \Psi_{> l} \theta_{>l}^{*} \right\|_{2}^{2} = \sum\limits_{i=1}^{n} \beta_{i}$. Since $ \{ x_{i} \}_{i=1}^{n}$ are i.i.d., we have $ \{ \beta_{i} \}_{i=1}^{n}$ are also i.i.d., with mean
    \begin{align}
        E_{x_{i}} \left( \beta_{i} \right) &= \sum_{k=l+1}^{\infty} \sum_{m=1}^{N(d,k)} \left( \theta_{k,m}^{*}\right)^{2} = \| \theta_{>l}^{*} \|_{2}^{2}, \notag
    \end{align}
    where we use the fact that $ \{\psi_{k,m} \}_{k \ge 0, m \ge 1}$ is an orthonormal basis. Then, using the Markov's inequality, we have 
    \begin{align}
        \mathrm{pr} \left( \sum\limits_{i=1}^{n} \beta_{i} > \frac{1}{\delta} n \| \theta_{>l}^{*} \|_{2}^{2} \right) \le \delta. \notag
    \end{align}
    We finish the proof.
\end{proof}

The following Lemma \ref{lemma cal of bias 1} and \ref{lemma cal of bias 2} calculate the convergence rates of two important quantities about the bias term.
\begin{lemma}\label{lemma cal of bias 1}
    Under the assumptions of Theorem \ref{theorem bias}, denote $\tilde{s} = \min\{s,2\}$, we have
    \begin{displaymath}
        \left\| \Sigma_{\le l}^{-1}  \theta^{*}_{\le l} \right\|_{2}^{2}  = \Theta \left( d^{(2-\tilde{s})l } \right).
    \end{displaymath}
    where $\Theta$ only involves constants depending on $s, \gamma, R_{\gamma}$ and $c_{0}$.
\end{lemma}

\begin{proof}
    In this proof, all the constants only depend on $s, \gamma, R_{\gamma}$ and $c_{0}$. By the definition of $ \Sigma_{\le l}, \theta^{*}_{\le l}$ in Section \ref{subsection more notiations}, we have
    \begin{equation}\label{eq proof le l-1}
        \left\| \Sigma_{\le l}^{-1}  \theta^{*}_{\le l} \right\|_{2}^{2} = \sum\limits_{k = 0}^{l} \mu_{k}^{-2} \sum\limits_{m = 1 }^{N(d,k)} (\theta^{*}_{k,m})^{2} = \sum\limits_{k = 0}^{l} \mu_{k}^{s-2} \mu_{k}^{-s} \sum\limits_{m = 1 }^{N(d,k)} (\theta^{*}_{k,m})^{2}.
    \end{equation}
    On the one hand, when $0 \le s \le 2$, using Lemma \ref{lemma inner eigen}, we have
    \begin{align}
        \textrm{Eq.}\eqref{eq proof le l-1} &\lesssim \mu_{l}^{s-2} \cdot  \sum\limits_{k = 0}^{l} \mu_{k}^{-s} \sum\limits_{m = 1 }^{N(d,k)} (\theta^{*}_{k,m})^{2} \lesssim d^{(2-s)l} R_{\gamma}, \notag
    \end{align}
    where we using Eq.\eqref{assumption source part 1} in Assumption \ref{assumption source condition} for the last inequality. When $s>2$, Lemma \ref{lemma inner eigen} implies that $ \mu_{k}^{s-2} \lesssim 1, k=0, 1, \ldots l$. Therefore, we have $\textrm{Eq.}\eqref{eq proof le l-1} \lesssim R_{\gamma}$, and we prove that $\left\| \Sigma_{\le l}^{-1}  \theta^{*}_{\le l} \right\|_{2}^{2}  = O \left( d^{(2-\tilde{s})l } \right). $

    On the other hand, when $0 \le s \le 2$, using Eq.\eqref{ass of fi} in Assumption \ref{assumption source condition}, we have
    \begin{align}
        \textrm{Eq.}\eqref{eq proof le l-1} \ge \mu_{l}^{s-2} \mu_{l}^{-s} \sum\limits_{m = 1 }^{N(d,l)} (\theta^{*}_{k,m})^{2} \gtrsim d^{(2-s)l} c_{0}. \notag
    \end{align}
    When $s > 2$, using Eq.\eqref{ass of fi} in Assumption \ref{assumption source condition} and the fact that $\mu_{k} \lesssim 1, ~\text{for any}~ k=0,1,\ldots$, we have $\textrm{Eq.}\eqref{eq proof le l-1} \gtrsim \sum\limits_{k=0}^{l} \sum\limits_{m = 1 }^{N(d,0)} (\theta^{*}_{k,m})^{2} \gtrsim c_{0} $. Therefore, we have $\textrm{Eq.}\eqref{eq proof le l-1} \lesssim R_{\gamma}$, and we prove that $\left\| \Sigma_{\le l}^{-1}  \theta^{*}_{\le l} \right\|_{2}^{2}  = \Omega \left( d^{(2-\tilde{s})l } \right). $
    We finish the proof.
\end{proof}

\begin{lemma}\label{lemma cal of bias 2}
    Under the assumptions of Theorem \ref{theorem bias}, we have
    \begin{displaymath}
        \| \theta^{*}_{>l} \|_{2}^{2}  = \Theta \left( d^{-(l+1)s} \right),
    \end{displaymath}
    where $\Theta$ only involves constants depending on $s, \gamma, R_{\gamma}$ and $c_{0}$.
\end{lemma}

\begin{proof}
    In this proof, all the constants only depend on $s, \gamma, R_{\gamma}$ and $c_{0}$. By the definition of $ \theta^{*}_{> l}$ in Section \ref{subsection more notiations}, we have
    \begin{equation}\label{eq proof ge l-1}
        \| \theta^{*}_{>l} \|_{2}^{2} = \sum\limits_{k=l+1}^{\infty} \mu_{k}^{s} \mu_{k}^{-s}  \sum\limits_{m = 1 }^{N(d,k)} (\theta^{*}_{k,m})^{2}.
    \end{equation}
    On the one hand, using $ \mu_{k} \lesssim \mu_{l+1} ~\text{for any}~ k \ge l+1$, which can be proved by Lemma \ref{lemma:monotone_of_eigenvalues_of_inner_product_kernels}.
    \begin{align}
        \textrm{Eq.}\eqref{eq proof ge l-1} \lesssim \mu_{l+1}^{s} \sum\limits_{k=l+1}^{\infty} \mu_{k}^{-s} \sum\limits_{m = 1 }^{N(d,k)} (\theta^{*}_{k,m})^{2} \lesssim \mu_{l+1}^{s} R_{\gamma} \lesssim d^{-(l+1)s} R_{\gamma}, \notag
    \end{align}
    where we use Eq.\eqref{assumption source part 1} in Assumption \ref{assumption source condition} and Lemma \ref{lemma inner eigen} for the last inequality.

    On the other hand, using Eq.\eqref{ass of fi} in Assumption \ref{assumption source condition} we have
    \begin{align}
        \textrm{Eq.}\eqref{eq proof ge l-1} \ge \mu_{l+1}^{s} \cdot \mu_{l+1}^{-s} \sum\limits_{m = 1 }^{N(d,l+1)} (\theta^{*}_{k,m})^{2} \gtrsim d^{-(l+1)s} c_{0}. \notag
    \end{align}
    We finish the proof.
\end{proof}

The following Lemma \ref{lemma B11} -- \ref{lemma B23} give several bounds that were used in the proof of Theorem \ref{theorem bias}. It can be checked that all the constants in the proof only depend on $s, \gamma, R_{\gamma}$, $c_{0}, c_{1}$, and $ c_{2}$.

\begin{lemma}\label{lemma B11}
Define $ B_{1,1} = \left\| \Sigma_{\le l} \Psi^{\top}_{\le l} K^{-1} \Psi_{\le l} \theta^{*}_{\le l} - \theta^{*}_{\le l}  \right\|_{2}^{2} $. Under assumptions of Theorem \ref{theorem bias}, we have
    \begin{displaymath}
        B_{1,1} = \Theta_{\mathbb{P}} \left( d^{(2-\tilde{s})l-2 \gamma} \right),
    \end{displaymath}
    where $\Theta_{\mathbb{P}}$ only involves constants depending on $s, \gamma, R_{\gamma}$, $c_{0}, c_{1}$ and $ c_{2}$.
\end{lemma}
\begin{proof}
Applying Sherman–Morrison–Woodbury formula to $ \left( \Sigma_{\le l}^{-1} + \Psi_{\le l}^{\top} K_{>l}^{-1} \Psi_{\le l}  \right)^{-1}$, we have
     \begin{align}\label{eq B11}
        B_{1,1} &= \left\| \left(\mathrm{I}_{B_{l}} - \Sigma_{\le l} \Psi^{\top}_{\le l} K^{-1} \Psi_{\le l}\right)  \theta^{*}_{\le l} \right\|_{2}^{2} \notag \\
        &= \left\| \left( \Sigma_{\le l} - \Sigma_{\le l} \Psi^{\top}_{\le l} \left( \Psi_{\le l} \Sigma_{\le l} \Psi_{\le l}^{\top} + K_{>l} \right)^{-1} \Psi_{\le l} \Sigma_{\le l}\right)  \Sigma_{\le l}^{-1}  \theta^{*}_{\le l} \right\|_{2}^{2} \notag \\
        &= \left\| \left( \Sigma_{\le l}^{-1} + \Psi_{\le l}^{\top} K_{>l}^{-1} \Psi_{\le l}  \right)^{-1}  \Sigma_{\le l}^{-1}  \theta^{*}_{\le l} \right\|_{2}^{2}.
    \end{align}
    Using Lemma \ref{lemma lambda max min} and Lemma \ref{lemma cal of bias 1}, we prove that $ B_{1,1} = \Theta_{\mathbb{P}} \left( n^{-2} \left\| \Sigma_{\le l}^{-1}  \theta^{*}_{\le l} \right\|_{2}^{2} \right) = \Theta_{\mathbb{P}} \left( d^{(2-\tilde{s})l-2 \gamma} \right).$
    We finish the proof.
\end{proof}

\begin{lemma}\label{lemma B12}
Define $B_{1,2} = \left\|  \Sigma_{\le l} \Psi^{\top}_{\le l} K^{-1} \Psi_{> l} \theta^{*}_{> l} \right\|_{2}^{2}$. Under assumptions of Theorem \ref{theorem bias}, we have
    \begin{displaymath}
       B_{1,2} = O_{\mathbb{P}} \left( d^{-(l+1)s} \right),
    \end{displaymath}
    where $O_{\mathbb{P}}$ only involves constants depending on $s, \gamma, R_{\gamma}$, $c_{0}, c_{1}$ and $ c_{2}$.

\end{lemma}
\begin{proof}
We first make the following calculation:
    \begin{align}\label{eq b22 mid}
        K^{-1} \Psi_{\le l} \Sigma_{\le l}^{1/2} &= \left( K_{>l} + \Psi_{\le l} \Sigma_{\le l} \Psi_{\le l}^{\top} \right)^{-1} \Psi_{\le l} \Sigma_{\le l}^{1/2} \notag \\
        &= \left[  K_{>l}^{-1} - K_{>l}^{-1} \Psi_{\le l} \Sigma_{\le l}^{1/2} \left(\mathrm{I}_{B_{l}} + \Sigma_{\le l}^{1/2}  \Psi_{\le l}^{\top} K_{>l}^{-1} \Psi_{\le l} \Sigma_{\le l}^{1/2} \right)^{-1}  \Sigma_{\le l}^{1/2}  \Psi_{\le l}^{\top} K_{>l}^{-1} \right] \Psi_{\le l} \Sigma_{\le l}^{1/2} \notag \\
        &= K_{>l}^{-1} \Psi_{\le l} \Sigma_{\le l}^{1/2} \left[ \mathrm{I}_{B_{l}} - \left( \mathrm{I}_{B_{l}} +  \Sigma_{\le l}^{1/2}  \Psi_{\le l}^{\top} K_{>l}^{-1} \Psi_{\le l} \Sigma_{\le l}^{1/2}\right)^{-1}  \left( \mathrm{I}_{B_{l}} +  \Sigma_{\le l}^{1/2}  \Psi_{\le l}^{\top} K_{>l}^{-1} \Psi_{\le l} \Sigma_{\le l}^{1/2} - \mathrm{I}_{B_{l}}\right)   \right] \notag \\
        &= K_{>l}^{-1} \Psi_{\le l} \Sigma_{\le l}^{1/2} \left( \mathrm{I}_{B_{l}} +  \Sigma_{\le l}^{1/2}  \Psi_{\le l}^{\top} K_{>l}^{-1} \Psi_{\le l} \Sigma_{\le l}^{1/2}\right)^{-1},
    \end{align}
    where we use the Sherman–Morrison–Woodbury formula for the second equality. Then we have
    \begin{align}\label{eq proof t2 subs}
        &\left( K^{-1} \Psi_{\le l} \Sigma_{\le l}^{1/2} \cdot \Sigma_{\le l} \cdot \Sigma_{\le l}^{1/2} \Psi_{\le l}^{\top}  K^{-1}  \right) \notag \\
        &=  K_{>l}^{-1} \Psi_{\le l} \Sigma_{\le l}^{1/2} \left( \mathrm{I}_{B_{l}} +  \Sigma_{\le l}^{1/2}  \Psi_{\le l}^{\top} K_{>l}^{-1} \Psi_{\le l} \Sigma_{\le l}^{1/2}\right)^{-1} \cdot \Sigma_{\le l} \cdot \left( \mathrm{I}_{B_{l}} +  \Sigma_{\le l}^{1/2}  \Psi_{\le l}^{\top} K_{>l}^{-1}  \Psi_{\le l} \Sigma_{\le l}^{1/2}\right)^{-1} \Sigma_{\le l}^{1/2}  \Psi_{\le l}^{\top} K_{>l}^{-1} \notag \\
        &= \frac{1}{n^{2}} K_{>l}^{-1} \Psi_{\le l} \left( \frac{\Sigma_{\le l}^{-1}}{n} + \frac{\Psi_{\le l}^{\top} K_{>l}^{-1} \Psi_{\le l}}{n}  \right)^{-2} \Psi_{\le l}^{\top} K_{>l}^{-1}.
    \end{align}
    Plugging Eq.\eqref{eq proof t2 subs} into the definition of $B_{1,2}$, we have
    \begin{align}\label{eq b12-1}
       B_{1,2} &= \left( \Psi_{>l} \theta^{*}_{>l}  \right)^{\top} \left( K^{-1} \Psi_{\le l} \Sigma_{\le l} \Sigma_{\le l} \Psi_{\le l}^{\top}  K^{-1}  \right) \left( \Psi_{>l} \theta^{*}_{>l}  \right) \notag \\
       &=  \frac{1}{n^{2}} \left( \Psi_{>l} \theta^{*}_{>l}  \right)^{\top} K_{>l}^{-1} \Psi_{\le l} \left( \frac{\Sigma_{\le l}^{-1}}{n} + \frac{\Psi_{\le l}^{\top} K_{>l}^{-1} \Psi_{\le l}}{n}  \right)^{-2} \Psi_{\le l}^{\top} K_{>l}^{-1} \left( \Psi_{>l} \theta^{*}_{>l}  \right) \notag \\
       &\asymp \frac{1}{n^{2}} \left( \Psi_{>l} \theta^{*}_{>l}  \right)^{\top} K_{>l}^{-1} \Psi_{\le l} \Psi_{\le l}^{\top} K_{>l}^{-1} \left( \Psi_{>l} \theta^{*}_{>l}  \right) \notag \\
       &\asymp  \frac{1}{n^{2}} \left( \Psi_{>l} \theta^{*}_{>l}  \right)^{\top} \Psi_{\le l} \Psi_{\le l}^{\top} \left( \Psi_{>l} \theta^{*}_{>l}  \right),
   \end{align}
   where we use Lemma \ref{lemma lambda max min} for the third line; use Eq.\eqref{eq k ge l in lemma} in Lemma \ref{lemma psi psi top} and $\kappa_{1} = \Theta(1)$ for the last line.

   Denote the SVD decomposition of $ \Psi_{\le l} $ as $ \Psi_{\le l} = n^{1/2} ~ O H V^{\top} $, where $ O \in \mathbb{R}^{n \times n}$ and $ V \in \mathbb{R}^{B_{l} \times B_{l}}$ are orthogonal matrices; $H = \left[ H_{*}; 0 \right] \in \mathbb{R}^{n \times B_{l}} $ and $H_{*} \in \mathbb{R}^{B_{l} \times B_{l}} $ is a diagonal matrix. Then Lemma \ref{lemma psi top psi} implies that $ H_{*} = I_{B_{l}} + \Delta_{h}$, with $ \| \Delta_{h} \|_{\mathrm{op}} = o_{\mathbb{P}}(1)$. Therefore, going back to Eq.\eqref{eq b12-1}, we have
   \begin{align}
       B_{1,2} &\asymp \frac{1}{n} \left( \Psi_{>l} \theta^{*}_{>l}  \right)^{\top}   \left( \frac{\Psi_{\le l} \Psi_{\le l}^{\top}}{n} \right) \left( \Psi_{>l} \theta^{*}_{>l}  \right) = \frac{1}{n} \left( \Psi_{>l} \theta^{*}_{>l}  \right)^{\top}   \left( O H V^{\top} V H^{\top} O^{\top} \right) \left( \Psi_{>l} \theta^{*}_{>l}  \right) \notag \\
       &= \frac{1}{n} \left( \Psi_{>l} \theta^{*}_{>l}  \right)^{\top}   \left( O \text{Diag}\left\{ H_{*}, \mathrm{0}_{n-B_{l}}\right\}  O^{\top} \right) \left( \Psi_{>l} \theta^{*}_{>l}  \right) \notag \\
       &\lesssim \frac{1}{n} \left( \Psi_{>l} \theta^{*}_{>l}  \right)^{\top} \left( \Psi_{>l} \theta^{*}_{>l}  \right) \notag \\
       &= O_{\mathbb{P}} \left(  \| \theta^{*}_{>l} \|_{2}^{2} \right) = O_{\mathbb{P}} \left( d^{-(l+1)s} \right), \notag
   \end{align}
where for the last line, we use Lemma \ref{lemma ge l l2}, Lemma \ref{lemma cal of bias 2} and the fact that $ B_{l} \asymp d^{l} \ll n$ (Eq.\eqref{eq Bl asymp}). We finish the proof.
\end{proof}

\begin{lemma}\label{lemma B22}
Define $B_{2,2} =  \left\| \Sigma_{> l} \Psi^{\top}_{> l} K^{-1} \Psi_{\le l} \theta^{*}_{\le l}  \right\|_{2}^{2}$. Under assumptions of Theorem \ref{theorem bias}, we have
    \begin{equation}\label{eq def of b22}
        B_{2,2} = o_{\mathbb{P}} \left( d^{(2-\tilde{s})l-2 \gamma} \right),
    \end{equation}
    where $o_{\mathbb{P}}$ only involves constants depending on $s, \gamma, R_{\gamma}$, $c_{0}, c_{1}$ and $ c_{2}$.
\end{lemma}

\begin{proof}
    Using the calculation Eq.\eqref{eq b22 mid} again and plugging it into Eq.\eqref{eq def of b22}, we have
    \begin{align}
        B_{2,2} &= \left\| \Sigma_{> l} \Psi^{\top}_{> l} \left( K^{-1} \Psi_{\le l} \Sigma_{\le l}^{1/2} \right) \Sigma_{\le l}^{-1/2}  \theta^{*}_{\le l}  \right\|_{2}^{2} \notag \\
        &= \left\| \Sigma_{> l} \Psi^{\top}_{> l}  K_{>l}^{-1} \Psi_{\le l} \Sigma_{\le l}^{1/2} \left( \mathrm{I}_{B_{l}} +  \Sigma_{\le l}^{1/2}  \Psi_{\le l}^{\top} K_{>l}^{-1} \Psi_{\le l} \Sigma_{\le l}^{1/2}\right)^{-1}  \Sigma_{\le l}^{-1/2}  \theta^{*}_{\le l}  \right\|_{2}^{2} \notag \\
        &= \left\| \frac{1}{n} \Sigma_{> l} \Psi^{\top}_{> l}  K_{>l}^{-1} \Psi_{\le l}  \left( \frac{\Sigma_{\le l}^{-1}}{n} + \frac{\Psi_{\le l}^{\top} K_{>l}^{-1} \Psi_{\le l}}{n}  \right)^{-1}  \Sigma_{\le l}^{-1}  \theta^{*}_{\le l}  \right\|_{2}^{2}. \notag
    \end{align}
    Using Lemma \ref{lemma lambda max min} again, we have
    \begin{align}
        B_{2,2} &\asymp  \left\| \frac{1}{n} \Sigma_{> l} \Psi^{\top}_{> l}  K_{>l}^{-1} \Psi_{\le l} \Sigma_{\le l}^{-1}  \theta^{*}_{\le l}  \right\|_{2}^{2} \notag \\
        &\le \left\| \Sigma_{>l} \right\|_{\mathrm{op}} \left\| \frac{1}{n} \Sigma_{> l}^{1/2} \Psi^{\top}_{> l}  K_{>l}^{-1} \Psi_{\le l} \Sigma_{\le l}^{-1}  \theta^{*}_{\le l}  \right\|_{2}^{2} \notag \\
        &\le \left\| \Sigma_{>l} \right\|_{\mathrm{op}} \frac{1}{n^{2}} \left(  \Sigma_{\le l}^{-1}  \theta^{*}_{\le l} \right)^{\top} \Psi_{\le l}^{\top} K_{>l}^{-1} \Psi_{>l} \Sigma_{> l} \Psi^{\top}_{> l}  K_{>l}^{-1} \Psi_{\le l} \left(\Sigma_{\le l}^{-1}  \theta^{*}_{\le l}\right) \notag \\
        &= n \left\| \Sigma_{>l} \right\|_{\mathrm{op}} \cdot \frac{1}{n^{2}} \left(  \Sigma_{\le l}^{-1}  \theta^{*}_{\le l} \right)^{\top} \cdot \frac{\Psi_{\le l}^{\top}  K_{>l}^{-1} \Psi_{\le l} }{n} \cdot \left(\Sigma_{\le l}^{-1}  \theta^{*}_{\le l}\right).
    \end{align}
    Lemma \ref{lemma inner eigen} shows that $n \left\| \Sigma_{>l} \right\|_{\mathrm{op}} \asymp n d^{-(l+1)} = o(1) $. Using Eq.\eqref{eq k ge l in lemma} in Lemma \ref{lemma psi psi top} and Lemma \ref{lemma psi top psi}, we have
    \begin{align}
        B_{2,2} = o_{\mathbb{P}}(1) \cdot \frac{1}{n^{2}} \left\| \Sigma_{\le l}^{-1}  \theta^{*}_{\le l} \right\|_{2}^{2} = o_{\mathbb{P}} \left( d^{(2-\tilde{s})l-2 \gamma} \right), \notag
    \end{align}
    where we use Lemma \ref{lemma cal of bias 1} for the lase equality. We finish the proof.
\end{proof}

\begin{lemma}\label{lemma B23}
Define $B_{2,3} = \left\|  \Sigma_{>l} \Psi^{\top}_{> l} K^{-1} \Psi_{> l} \theta^{*}_{> l} \right\|_{2}^{2}$. Under assumptions of Theorem \ref{theorem bias}, we have
    \begin{displaymath}
        B_{2,3} = o_{\mathbb{P}} \left( d^{-(l+1)s} \right),
    \end{displaymath}
    where $o_{\mathbb{P}}$ only involves constants depending on $s, \gamma, R_{\gamma}$, $c_{0}, c_{1}$ and $ c_{2}$.
\end{lemma}

\begin{proof}
First, we have
    \begin{align}
        B_{2,3} &= \left\| \Sigma_{>l}^{1/2} \Sigma_{>l}^{1/2} \Psi^{\top}_{> l} K^{-1} \Psi_{> l} \theta^{*}_{> l} \right\|_{2}^{2} \notag \\
        &\le \left\| \Sigma_{>l} \right\|_{\mathrm{op}}  \left( \Psi_{> l}  \theta^{*}_{> l} \right)^{\top} K^{-1}   \Psi_{> l} \Sigma_{>l} \Psi_{> l}^{\top} K^{-1} \left( \Psi_{> l}  \theta^{*}_{> l}\right) \notag \\
        &\le \left\| \Sigma_{>l} \right\|_{\mathrm{op}}  \left( \Psi_{> l}  \theta^{*}_{> l} \right)^{\top} K^{-1} \left( \Psi_{> l}  \theta^{*}_{> l}\right) \notag \\
        &\le n \left\| \Sigma_{>l} \right\|_{\mathrm{op}} \cdot \lambda_{\text{max}}(K_{>l}^{-1}) \cdot \frac{1}{n} \left\| \Psi_{> l}  \theta^{*}_{> l} \right\|_{2}^{2}. \notag
    \end{align}
    where we use the fact that $K- \Psi_{> l} \Sigma_{>l} \Psi_{> l}^{\top} $ is a positive definite matrix for the second inequality. Using Lemma Eq.\eqref{eq k ge l in lemma} in \ref{lemma psi psi top}, Lemma \ref{lemma ge l l2} and the fact that $n \left\| \Sigma_{>l} \right\|_{\mathrm{op}} \asymp n d^{-(l+1)} = o(1) $, we prove that
    \begin{displaymath}
        B_{2,3} = o_{\mathbb{P}}(1) \cdot \| \theta^{*}_{>l} \|_{2}^{2} = o_{\mathbb{P}} \left( d^{-(l+1)s} \right),
    \end{displaymath}
    where we use Lemma \ref{lemma cal of bias 2} for the last equality. We finish the proof.
\end{proof}
The main idea of Lemma \ref{lemma B22}, \ref{lemma B23} is from \cite[Lemma 14]{barzilai2023generalization}. By comparison, we prove the $ o_{\mathbb{P}}$ bounds instead of $ O_{\mathbb{P}}$ by making more elaborate calculations under the assumption of the inner product kernel.\\

The following identity is borrowed from Lemma 6 in \cite{aerni2022strong} and Lemma 20 in \cite{bartlett2020benign}, whose proof is mainly based on the Sherman–Morrison–Woodbury formula.
\begin{lemma}\label{lemma trans in V2}
Recall that $l = \lfloor \gamma \rfloor$. It always holds that
    \begin{displaymath}
        \Sigma_{\le l} \Psi_{\le l}^{\top} K^{-2} \Psi_{\le l} \Sigma_{\le l} = \left( \Sigma_{\le l}^{-1} + \Psi_{\le l}^{\top} K_{>l}^{-1} \Psi_{\le l} \right)^{-1}  \Psi_{\le l}^{\top} K_{>l}^{-2} \Psi_{\le l}  \left( \Sigma_{\le l}^{-1} + \Psi_{\le l}^{\top} K_{>l}^{-1} \Psi_{\le l} \right)^{-1}.
    \end{displaymath}
\end{lemma}


\section*{Acknowledgments}
This work was supported in part by the National Natural Science Foundation of China and Beijing Natural Science Foundation.

\bibliographystyle{plainnat}
\bibliography{references}

\end{document}